\documentclass{article}

\usepackage[preprint]{neurips_2026}

\usepackage[utf8]{inputenc} 
\usepackage[T1]{fontenc}    
\usepackage{hyperref}       
\usepackage{url}            
\usepackage{booktabs}       
\usepackage{amsfonts}       
\usepackage{nicefrac}       
\usepackage{microtype}      
\usepackage{xcolor}         
\usepackage{wrapfig}

\usepackage{microtype}
\usepackage{graphicx}
\usepackage{subcaption}
\usepackage{booktabs} 
\usepackage{booktabs}
\usepackage{siunitx}
\usepackage{xcolor,colortbl}
\usepackage{booktabs}
\usepackage[table]{xcolor}
\usepackage{tabularx}
\usepackage{makecell}
\usepackage{marvosym}
\usepackage{array} 
\usepackage{booktabs}

\usepackage{multirow}

\usepackage{amsmath}
\usepackage{amssymb}
\usepackage{mathtools}
\usepackage{amsthm}
\usepackage{colortbl}
\usepackage{cleveref}
\usepackage{graphicx}
\usepackage{caption}
\usepackage{subcaption}
\usepackage{booktabs}

\definecolor{upColor}{RGB}{17,138,21}
\definecolor{downColor}{RGB}{174,36,67}

\definecolor{oursblue}{RGB}{225,240,255}

\theoremstyle{plain}
\newtheorem{theorem}{Theorem}[section]

\newtheorem{lemma}[theorem]{Lemma}
\newtheorem{corollary}[theorem]{Corollary}
\theoremstyle{definition}

\theoremstyle{remark}
\newtheorem{remark}[theorem]{Remark}

\usepackage[textsize=tiny]{todonotes}

\title{Hyper-DP3: Frequency-Aware Right-Sizing of 3D Diffusion Policies for Visuomotor Control}

\makeatletter
\newcommand{\afffootnote}[1]{%
\begingroup
\renewcommand\thefootnote{}%
\footnote{#1}%
\addtocounter{footnote}{-1}%
\endgroup
}
\makeatother

\author{%
Jinhao Zhang$^{1}$\thanks{Equal contribution.}
\And
Zhexuan Zhou$^{1}$\footnotemark[1]
\And
Huizhe Li$^{1}$
\And
Yichen Lai$^{1}$
\AND
Wenlong Xia$^{1}$
\And
Haoming Song$^{2}$
\And
Youmin Gong$^{1}$
\And
Jie Mei$^{1}$\thanks{Corresponding author: \texttt{jmei@hit.edu.cn}}
}

\begin{document}

\maketitle

\afffootnote{$^{1}$ Harbin Institute of Technology, Shenzhen}
\afffootnote{$^{2}$ Shanghai Jiao Tong University}

\begin{abstract}
Diffusion-based visuomotor policies perform well in robotic manipulation, yet current methods still inherit image-generation-style decoders and multi-step sampling. We revisit this design from a frequency-domain perspective. Robot action trajectories are highly smooth, with most energy concentrated in a few low-frequency discrete cosine transform modes. Under this structure, we show that the error of the optimal denoiser is bounded by the low-frequency subspace dimension and residual high-frequency energy, implying that denoising error saturates after very few reverse steps. This also suggests that action denoising requires a much simpler denoising model than image generation. Motivated by this insight, we propose Hyper-DP3 ({\bf HDP3}), a pocket-scale 3D diffusion policy with a lightweight Diffusion Mixer decoder that supports two-step DDIM inference.  Our synthetic experiments validate the theory and support the sufficiency of two-step denoising. Futhermore, across RoboTwin2.0, Adroit, MetaWorld, and real-world tasks, HDP3 achieves state-of-the-art performance with fewer than 1\% of the parameters of prior 3D diffusion-based policies and substantially lower inference latency.
\end{abstract}

\section{Introduction}
Learning from demonstration provides a practical way to acquire visuomotor manipulation skills directly from expert data\citep{argall2009survey,osa2018algorithmic}, but standard behavior cloning often struggles with the multimodal nature of robotic actions, leading to mode averaging and compounding errors under distribution shift\citep{ross2011reduction}. Diffusion-based policies address this limitation by modeling a conditional denoising process over action trajectories, and have shown strong performance in manipulation tasks\citep{chi2025diffusion}. Recent work has further extended this paradigm from 2D image observations to 3D point clouds, improving geometric robustness and transferability\citep{ze20243d}.

However, current 3D diffusion-based policies still largely inherit their design from image generation\citep{ze20243d}. They typically combine a lightweight point-cloud encoder with a heavy decoder backbone, often a conditional U-Net\citep{ronneberger2015u} or DiT\citep{peebles2023scalable}, and require multiple denoising steps at inference time. This is reasonable for reconstructing dense high-frequency image details, but robot action generation is fundamentally different: the target is a short, temporally smooth trajectory rather than a perceptual signal. This mismatch suggests that current 3D diffusion policies may be substantially overdesigned in both decoder capacity and denoising steps.

Our starting point is the frequency structure of robot trajectories. We find that action energy in frequency domain is highly concentrated in the lowest discrete cosine transform (DCT) modes, with the first two modes accounting for nearly all of the total energy. Building on this observation, we provide a frequency-domain theoretical analysis showing that, when trajectory energy is concentrated in a shallow frequency-domain subspace, only a very small number of sampling steps is sufficient to drive the decoding error to a low level, thereby having negligible effect on the closed-loop execution of the trajectory. As a result, for smooth robot trajectories, performance should saturate after very few reverse steps (as observed in our experiments, two steps are already sufficient), and later denoising iterations offer diminishing returns.

Besides, recent theory on diffusion models shows that, when data are concentrated on a low-dimensional subspace, the complexity of score approximation is governed primarily by the intrinsic, rather than ambient, dimension of the data\citep{chen2023score}. From this perspective, the strong concentration of robot actions in the frequency domain suggests that the score field to be learned for action denoising is substantially simpler than that of high-dimensional perceptual signals. Therefore, a large image-generation-style decoder may be unnecessary; what is needed instead is efficient fusion of temporal, channel-wise, and conditional information. Based on this principle, we propose Hyper-DP3 (\textbf{HDP3}), a pocket-scale 3D diffusion policy that replaces the heavy conditional decoder of prior methods with a lightweight \textbf{Diffusion Mixer (DiM)} decoder. With standard DDIM sampling\citep{song2020denoising}, HDP3 supports two-step inference without consistency distillation or MeanFlow training. Across three simulation benchmarks—RoboTwin2.0, Adroit, and MetaWorld—it achieves state-of-the-art performance with fewer than 1\% of the parameters of prior methods and substantially lower inference latency, while real-world experiments further validate its practicality beyond simulation. Our contributions are summarized as follows:
\begin{itemize}
  \item We show that robot action trajectories are strongly low-frequency-dominant, revealing a mismatch in existing 3D diffusion policies, which still rely on heavy image-generation-style decoders and multi-step denoising.
  \item We provide a theoretical frequency-domain analysis showing that, for such trajectories, denoising error saturates quickly, so very few reverse steps are sufficient for effective action generation.
  \item We propose Hyper-DP3 (\textbf{HDP3}), a pocket-scale 3D diffusion policy with a lightweight \textbf{Diffusion Mixer (DiM)} decoder that supports efficient two-step DDIM inference.
  \item We validate the analysis and design on synthetic data, three simulation benchmarks, and real-world tasks, achieving state-of-the-art performance with fewer than 1\% of the parameters of prior methods and substantially lower inference latency.
\end{itemize}

\section{Related Work}
\subsection{Diffusion Models for Visuomotor Control.} Traditional Behavior Cloning (BC) often struggles with multimodal action distributions, leading to mode-averaging artifacts that result in suboptimal or unsafe behaviors in manipulation tasks. Diffusion Policy\citep{chi2025diffusion} addresses this fundamental limitation by formulating policy learning as a conditional denoising process over the action space. By learning the gradient of the action distribution, it effectively captures complex, multimodal human behaviors and provides superior stability in high-dimensional manipulation tasks compared to explicit regression policies. This probabilistic formulation has demonstrated remarkable expressiveness in capturing the inherent stochasticity of human demonstrations, outperforming prior generative approaches such as Implicit Behavior Cloning\citep{florence2022implicit} and Conditional VAEs\citep{sohn2015learning} across diverse benchmarks. Furthermore, the iterative refinement mechanism inherent to diffusion models enables the generation of temporally coherent action sequences, which is critical for contact-rich manipulation scenarios. More recently, generative robot policies have also been extended to the vision-language-action regime \citep{black2024pi0, physicalintelligence2025pi05, song2025hume}.

\subsection{From 2D to 3D Representations.} While early diffusion policies relied predominantly on 2D RGB images as visual observations, such representations lack robustness against lighting variations, viewpoint changes, and domain shifts between simulation and real-world environments. Recent works leverage 3D point clouds processed by geometric encoders like PointNet++\citep{qi2017pointnet++} to extract view-invariant spatial features that better capture object geometry and spatial relationships. 3D Diffusion Policy (DP3) \citep{ze20243d} integrates these efficient 3D representations with diffusion models, achieving state-of-the-art data efficiency and generalization capabilities in few-shot imitation learning settings. Complementary approaches such as Act3D\citep{gervet2023act3d} and PerAct\citep{shridhar2023perceiver} have further demonstrated the benefits of 3D-aware representations for language-conditioned manipulation. However, existing architectures typically pair lightweight point encoders with computationally heavy U-Net decoders for the diffusion backbone, creating a significant parameter redundancy and inference efficiency bottleneck that limits real-time deployment on resource-constrained robotic platforms.

\subsection{Inference Acceleration.} To overcome the prohibitively high latency of iterative denoising during deployment, recent research has focused on principled approaches to sampling acceleration without sacrificing action quality. One-Step Diffusion Policy (OneDP)\citep{wang2025onestep} employs knowledge distillation techniques to compress the multi-step diffusion process into a single forward pass, achieving real-time control frequencies of 62Hz suitable for reactive manipulation. Consistency Policy\citep{prasad2024consistency} takes an alternative approach by enforcing self-consistency constraints along probability flow ordinary differential equation (ODE) trajectories, enabling high-quality action generation in just 1-2 denoising steps while preserving the multimodal expressiveness of the original diffusion formulation. Drawing inspiration from recent advances in generative modeling, FlowPolicy\citep{zhang2025flowpolicy} utilizes Consistency Flow Matching\citep{yang2024consistency} to learn straight-line ODE paths between noise and data distributions, enabling faster, more numerically stable inference. MP1\citep{sheng2026mp1} employs MeanFlow\citep{geng2025mean} for one-step generation without timestep conditioning. These acceleration techniques collectively represent a promising direction toward bridging the gap between the representational power of diffusion-based policies and the stringent latency requirements of closed-loop robotic control.

\section{Problem Setup and Policy Overview}
We learn an end-to-end visuomotor policy $\pi_\theta:\mathcal{O}\rightarrow\mathcal{A}$ from expert demonstrations of observation--action pairs, where each observation consists of a single-view, robot-centric point cloud $\mathbf{P}$ and robot proprioceptive state $\mathbf{s}$. Following the DDPM formulation\citep{ho2020denoising}, we model action generation as a conditional denoising process. Given a clean action trajectory $x_0$, the forward process corrupts it at timestep $t$ as
\begin{equation}
x_t=\alpha_t x_0+\sigma_t \epsilon, \qquad \epsilon\sim\mathcal{N}(\mathbf{0},\mathbf{I}),
\end{equation}
where $\alpha_t$ and $\sigma_t$ are determined by the noise schedule. Our model adopts the $x_0$-prediction parameterization\citep{ramesh2022hierarchical} and directly predicts the clean trajectory $\hat{x}_0$ from the
noisy input $x_t$.

As shown in Fig.~\ref{fig:arch}, our framework consists of a compact observation encoder and a lightweight diffusion decoder. Following DP3\citep{ze20243d}, we first downsample $\mathbf{P}$ using farthest point
sampling and encode the sampled points with an MLP followed by max pooling to obtain a compact 3D feature $\mathbf{z}$. We then linearly project the robot state and add it to $\mathbf{z}$ to form the conditioning context:
\begin{equation}
C=\mathbf{z}+\mathrm{Linear}(\mathbf{s}).
\end{equation}
Conditioned on $C$ and the diffusion timestep $t$, the decoder takes the noisy trajectory $x_t$ as input, projects it into a sequence of tokens, processes them through $K$ DiM blocks, and maps the fused features to the
final prediction $\hat{x}_0$. Details of the DiM block are provided in Sec.~\ref{sec:dim}.

During training, we optimize the reconstruction objective
\begin{equation}
\mathcal{L}=\mathbb{E}_{(x_0,\mathbf{P},\mathbf{s}),t,\epsilon}
\left[\|\hat{x}_0-x_0\|_2^2\right].
\end{equation}
At inference time, we use DDIM sampling\citep{song2020denoising} to accelerate denoising. For highly smooth trajectory data, we find that, in practice, two denoising steps are sufficient without requiring additional
consistency distillation techniques\citep{prasad2024consistency,wang2024one,zhang2025flowpolicy}. The corresponding theoretical justification is provided in Sec.~\ref{sec:2s}.

\section{Method}
\subsection{Motivation} \label{sec:motivation}

Most existing diffusion-based visuomotor policies\citep{chi2025diffusion,ze20243d} rely on multi-step DDPM\citep{ho2020denoising} or DDIM\citep{song2020denoising} denoising schedules together with large U-Net\citep{ronneberger2015u} or DiT\citep{peebles2023scalable} decoders. However, prior work has largely overlooked a key property of robot trajectory data: these trajectories are highly smooth in the time domain, and thus their energy is concentrated primarily in low-frequency bands in the frequency domain as shown in \citep{tan2024fourier}. Exploiting this property, we show both theoretically and empirically that, for robot trajectory data, just \textbf{two DDIM denoising steps} are sufficient for the decoding accuracy (measured by MSE) and rollout success rate to nearly saturate, \textbf{without} requiring additional sampling acceleration techniques such as consistency distillation\cite{song2023consistency} or MeanFlow\citep{geng2025mean}. Furthermore, motivated by the fact that action data inherently lie on a low-dimensional manifold, and by recent theory showing that score approximation in diffusion models depends on the intrinsic rather than ambient dimension of the data\citep{chen2023score}, we propose a much lighter decoder architecture that uses less than 1\% of the parameters of prior methods while surpassing previous state-of-the-art performance.

\begin{wrapfigure}{r}{0.52\linewidth}
    \centering
    \vspace{-0.5\baselineskip}
    \includegraphics[width=\linewidth]{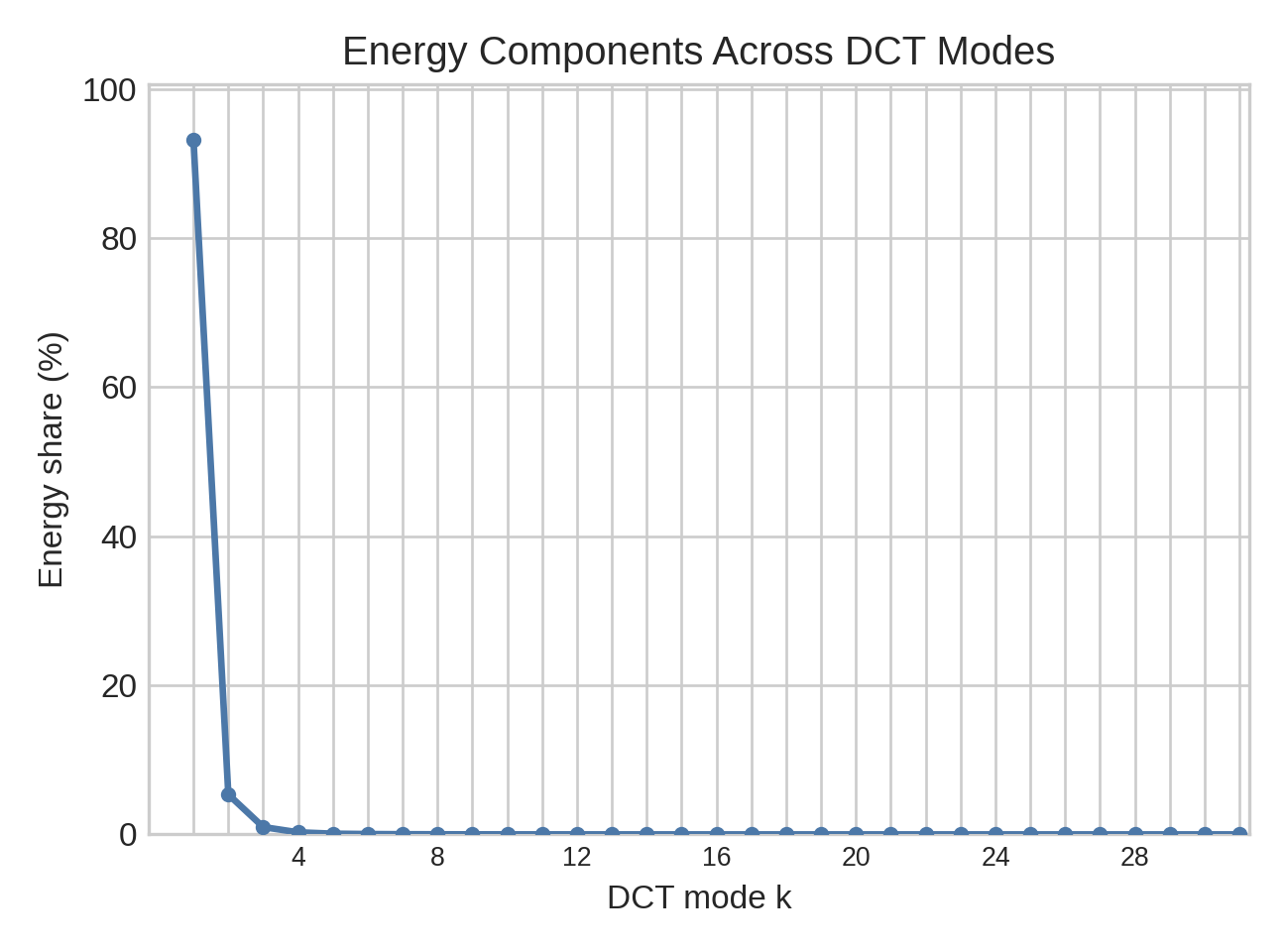}
    \caption{
    Frequency structure of action trajectories.
    }
    \label{fig:dct_stats}
    \vspace{-0.5\baselineskip}
\end{wrapfigure}
To further quantify the low-frequency-dominant nature of robot trajectories, we apply the discrete cosine transform (DCT) to map them into the frequency domain and measure the energy proportion in each frequency band. Specifically, we collect 2,500 episodes across all 50 tasks in RoboTwin2.0 and segment them into approximately 20,000 trajectories with a sequence length of $T=32$. We then apply the DCT and compute the energy ratio of each DCT mode, as shown in Figure~\ref{fig:dct_stats}. As can be seen, the first two DCT modes account for 98.5\% of the total energy, with the vast majority concentrated in the $k=1$ mode alone (93.2\%). This indicates that robot trajectory data lies on a low-dimensional manifold in the frequency domain and is highly concentrated in the low-frequency regime. More details on the frequency-domain decomposition and the energy distribution for different tasks are provided in Appendix~\ref{app:freq}.

\subsection{Two-Step Denoising Suffices for Robot Action Data} \label{sec:2s}

Prior work has provided both theoretical and empirical evidence that diffusion models tend to first generate low-frequency structure during sampling and then refine high-frequency details\citep{falck2025fourier}. Intuitively, for robot trajectories, where low-frequency components are overwhelmingly dominant, we expect the low-frequency portion to be reconstructed quickly in the early sampling steps. Subsequent refinement of the high-frequency components, which are mainly attributed to sensor noise and other minor perturbations\citep{tan2024fourier}, becomes unnecessary. As a result, only a small number of decoding steps should be sufficient to reach saturation performance.

To formalize this intuition, we study the denoising error in the frequency domain. Let $x_0$ denote the original data, and let $y_0$ denote its frequency-domain representation obtained through an orthogonal DCT. For analytical convenience, we assume that $y_0=\mathbf{U}x_0$ is distributed as a zero-mean Gaussian with covariance $\mathbf{\Sigma}$. Let $\mathbf{P}_L = \operatorname{diag}(1,\ldots,1,0,\ldots,0)$ be a diagonal matrix with the first $m$ diagonal entries equal to 1 and the remaining entries equal to 0, and let $\mathbf{P}_H = \mathbf{I} - \mathbf{P}_L$. The low-frequency-dominant nature of trajectory data can then be expressed as follows:
\begin{equation}
    \operatorname{Tr}(\mathbf{P}_H \mathbf{\Sigma}) \leq \eta \operatorname{Tr}(\mathbf{\Sigma})
\end{equation}
where $\operatorname{Tr}(\cdot)$ denotes the matrix trace, and $\eta$ denotes the fraction of energy contained in the high-frequency components. Consider the forward noising process $x_t = \alpha_t x_0 + \sigma_t \epsilon$. Then, as shown in~\citep{chung2022diffusion}, the optimal denoiser is given by the posterior expectation $x_0^*=\mathbb{E}[x_0|x_t]$. Because $\mathbf{U}$ is orthogonal, the noising process and the optimal estimator admit the same form in the frequency domain:
\begin{gather}
    y_t = \alpha_t y_0 + \sigma_t \xi, \quad \xi \sim \mathcal{N}(0, {\bf I}) \\
    y_0^*= \mathbf{U}x^*_0 = \mathbf{U}\mathbb{E}[x_0|x_t] = \mathbb{E}[\mathbf{U}x_0|\mathbf{U}x_t] = \mathbb{E}[y_0|y_t] \label{eq:optimal_freq}
\end{gather}
The error of this optimal estimator in the frequency domain satisfies the following theorem.
\begin{theorem}[Frequency-Domain Error Bound of the Optimal Estimator] \label{th:1}
    Under the above assumptions, the optimal estimation errors in the low- and high-frequency components satisfy, respectively:
    \begin{gather}
        e_L = \frac{1}{n}\mathbb{E}\|\mathbf{P}_L\left(y_0-\mathbb{E}[y_0|y_t]\right)\|^2 \leq \frac{m}{n}\frac{\sigma_t^2}{\alpha_t^2} \\
        e_H = \frac{1}{n}\mathbb{E}\|\mathbf{P}_H\left(y_0-\mathbb{E}[y_0|y_t]\right)\|^2 \leq \frac{\eta}{n} \operatorname{Tr}(\mathbf{\Sigma})
    \end{gather}
     where $n$ is the number of frames in the predicted trajectory, $e_L$ and $e_H$ denote the mean squared errors (MSEs) of the low- and high-frequency components, respectively. Consequently, the total error satisfies:
     \begin{equation} \label{eq:cov}
         e = \frac{1}{n}\mathbb{E}\|x_0-\mathbb{E}[x_0|x_t]\|^2 = \frac{1}{n}\mathbb{E}\|y_0-\mathbb{E}[y_0|y_t]\|^2 = e_L +e_H \leq \frac{m}{n}\frac{\sigma_t^2}{\alpha_t^2} + \frac{\eta}{n} \operatorname{Tr}(\mathbf{\Sigma})
     \end{equation} 
     The relative error is bounded by:
     \begin{equation} 
         \hat{e} = \frac{e}{\operatorname{Tr}({\bf \Sigma})} \leq \frac{m}{n\operatorname{Tr}({\bf \Sigma})}\frac{\sigma_t^2}{\alpha_t^2} + \frac{\eta}{n}
     \end{equation} 
\end{theorem}
\begin{proof}
    See App.~\ref{app:proof}.
\end{proof}
\begin{remark}
    The above theorem shows that {\bf the more concentrated the energy of the data is in the DCT spectrum, i.e., the smaller $\frac{m}{n}$ and $\eta$, the smaller the average per-frame error of the optimal estimator}. 
\end{remark}
For a $T$-step DDIM sampler, its final output can be interpreted as a one-step estimate of $x_0$ starting from $t=\frac{1}{T}$. This yields the following corollary for the MSE of $T$-step sampling.
\begin{corollary}
    The upper bound on the relative error $\hat{e}_T$ of a $T$-step DDIM sampler can be estimated as follows:
    \begin{equation} \label{eq:total_bound}
        \hat{e}_T \lesssim \frac{m}{n\operatorname{Tr}({\bf \Sigma})}\frac{\sigma_{T^{-1}}^2}{\alpha_{T^{-1}}^2} + \frac{\eta}{n}
    \end{equation} 
\end{corollary}
As a concrete example, substituting the data statistics from Sec.~\ref{sec:motivation}, namely $m=1$, $n=32$, and $\eta=0.07$, into Eq.~\eqref{eq:total_bound}, and replacing $\Sigma$ with the empirical covariance matrix of the trajectory in the frequency domain under the cosine noise schedule, we obtain the relative error bound $\hat{e}_T \leq 0.0023 \frac{\sigma_{T^{-1}}^2}{\alpha_{T^{-1}}^2} + 7\times10^{-5}$. Evaluating this bound at $T=2$ gives an approximate upper bound of $\hat{e}_2 \lesssim 0.25\%$, suggesting that for robot action data, two denoising steps already introduce only negligible approximation error. In addition, we find that, for low-frequency-dominant data, the error under few-step decoding mainly comes from the high-frequency band, and this component can be further suppressed by the low-level controller, making the resulting closed-loop performance nearly indistinguishable from that of multi-step decoding (See App.~\ref{app:toy} for details). 
\begin{remark}
    We can also explain the effectiveness of few-step sampling from the perspective of reduced crossings among interpolation trajectories: Eq.~\ref{eq:cov} is in fact the posterior variance $\operatorname{Var}(x_0|x_t)$, and a smaller value indicates a more concentrated and stable prediction target. As a result, the sampling trajectory follows a more nearly unique deterministic direction, making few-step sampling more feasible\citep{yang2026stable,liu2022flow}.
\end{remark}
The synthetic data experiments in Sec.~\ref{sec:sde} and the rollout evaluations in Sec.~\ref{sec:aba} further support that two-step decoding is sufficient in most cases.



\subsection{Lightweight Decoder for Minimal Denoising} \label{sec:dim}
Recent theory on diffusion models shows that, when data are concentrated on a low-dimensional subspace, the complexity of score approximation is governed primarily by the intrinsic, rather than ambient, dimension of the data\citep{chen2023score}. Combined with the strong low-frequency concentration of robot action trajectories observed above, this suggests that the score field to be learned for action denoising is substantially simpler than that of high-dimensional perceptual signals. Therefore, instead of relying on a heavy high-capacity decoder, we seek a lightweight architecture that is better matched to the intrinsic complexity of robot action data. Specifically, we adopt MLP-Mixer\citep{tolstikhin2021mlp} as a simple yet effective decoder backbone. Using only lightweight transpose operations and MLP-based projections, MLP-Mixer enables efficient interaction across both temporal and channel dimensions while maintaining stable training. As shown in Fig.~\ref{fig:arch}, given the token sequence $H$ obtained by projecting the noisy trajectory, we first transpose it and feed it into a temporal MLP to aggregate information along the time dimension. We then transpose it back and add a residual connection. Next, we apply the same procedure along the channel dimension, yielding a representation that integrates both temporal and channel-wise information:
\begin{equation}
    H = H+{\rm MLP}(H^{\rm T})^{\rm T};\quad H = H + {\rm MLP}(H)
\end{equation}
Finally, we inject the conditioning information into the model using a FiLM layer combined with a gated residual connection:
\begin{gather}
    \alpha, \beta, g = {\rm MLP}(H, C) \\
    \tilde{H} = H \odot \alpha + \beta \\
    H = \tilde{H} \odot g  + H
\end{gather}
After stacking $K$ such blocks, we obtain the final tokens in which the conditioning information is fully mixed, which are then used to predict the denoised actions:
\begin{equation}
    \hat{x}_0 = {\rm Linear}(H)
\end{equation}
Unlike U-Net and DiT, which rely on large intermediate dimensions to improve denoising capacity, our design substantially enhances the efficiency of information fusion, allowing us to surpass previous state-of-the-art performance using {\bf less than 1\%} of the parameters of prior methods.

\begin{figure}[t]
    \centering
    \includegraphics[width=\linewidth]{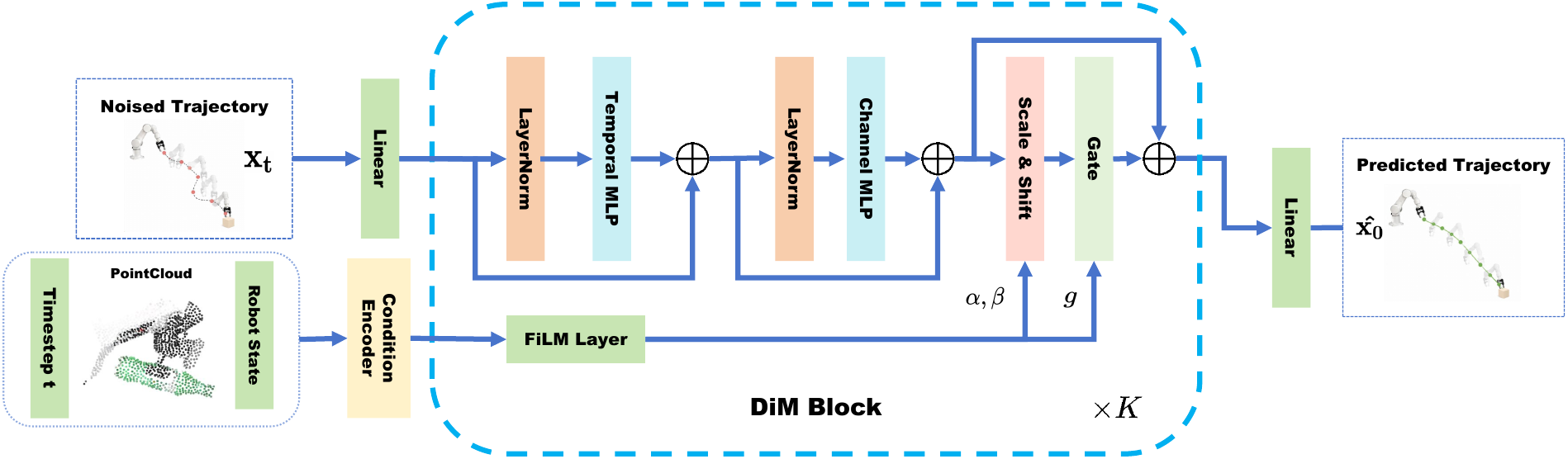}
    \caption{\textbf{Overall architecture of the Proposed Method.} 
      In the figure, $\rm T$ denotes transpose. We adopt the efficient point-cloud encoder from DP3\citep{ze20243d} and stacks $K$ DiM blocks as the decoder. Each DiM block is built upon an MLP-Mixer style\citep{tolstikhin2021mlp} architecture, enabling efficient information fusion with a small parameter budget, thereby improving decision-making performance.}
    \label{fig:arch}
\end{figure}

\section{Synthetic Data Experiments} \label{sec:sde}
\captionsetup{font=footnotesize}
\begin{figure}[htbp]
    \centering
    \begin{minipage}{0.28\textwidth}
        \centering
        \includegraphics[width=\linewidth]{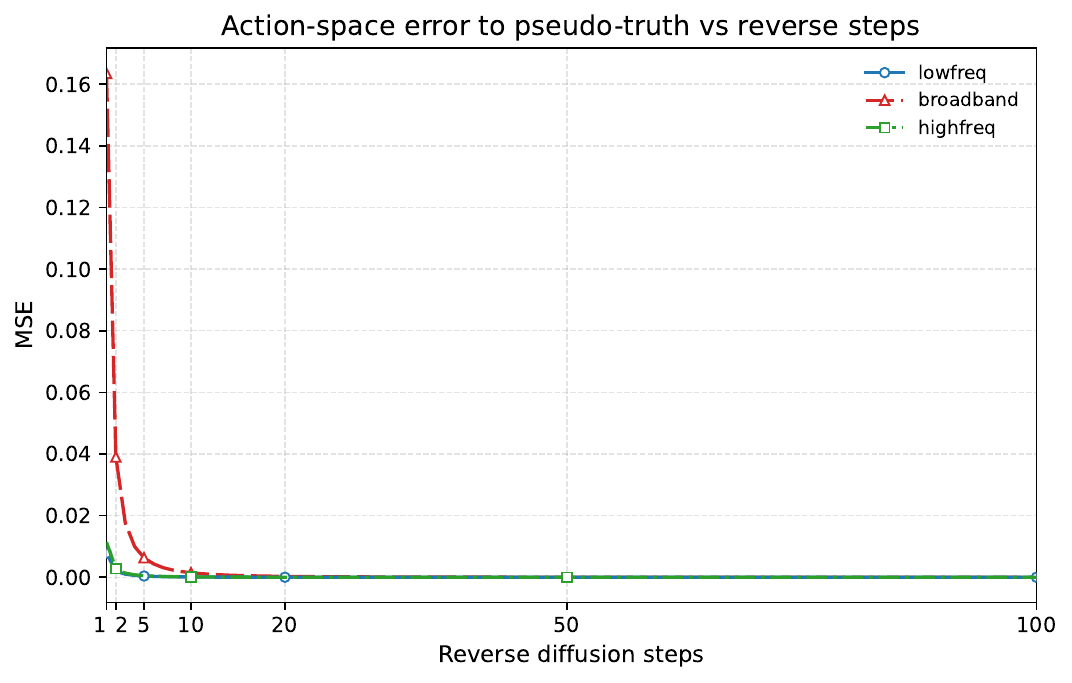}
        \caption{\scriptsize MSE Under Different NFEs}
        \label{fig:mse}
    \end{minipage}
    \hfill
    \begin{minipage}{0.70\textwidth}
        \centering
        \includegraphics[width=\linewidth]{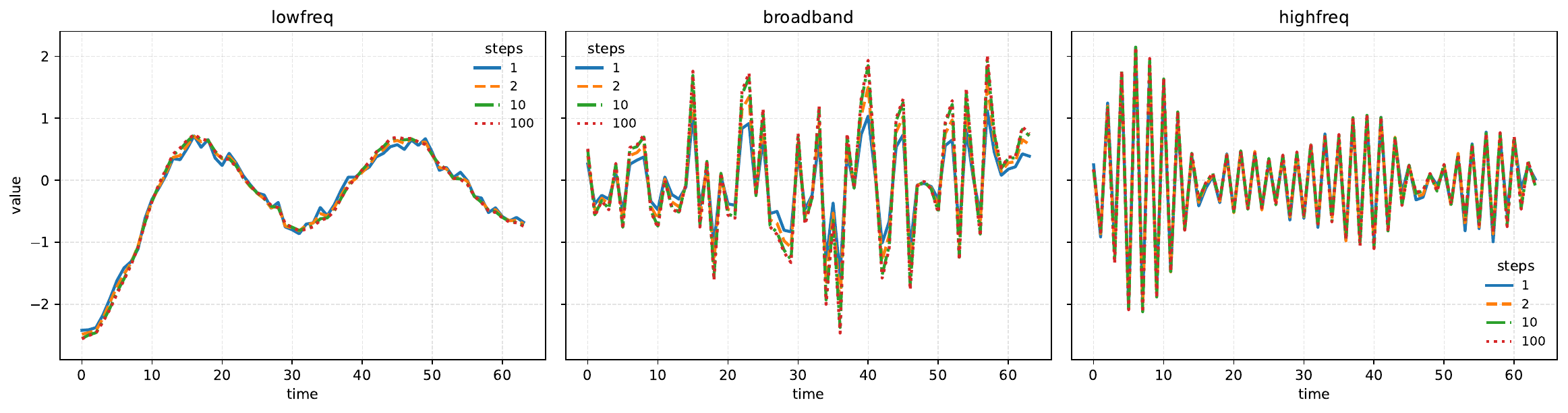}
        \caption{\scriptsize Examples of Decoded Trajectories}
        \label{fig:example}
    \end{minipage}
\end{figure}

To validate the theory in Sec.~\ref{sec:2s}, we design a corresponding synthetic-data experiment. Specifically, we synthesize three classes of trajectories: a low-frequency-dominant (\texttt{lowfreq}) dataset, where the first 6 out of 64 DCT modes account for 99\% of the total energy; a \texttt{broadband} dataset, where energy is evenly distributed across all 64 modes; and a high-frequency-dominant (\texttt{highfreq}) dataset, where the last 6 out of 64 DCT modes account for 99\% of the total energy. We normalize all datasets to a common scale to ensure a fair comparison in training and evaluation. We then train identical CNN-based diffusion models on these data and evaluate the MSE under different numbers of decoding (reverse) steps, taking the 100-step decoding result as the reference. Detailed
descriptions of data synthesis, training and evaluation protocols, and additional results are provided in App.~\ref{app:toy}.

The results are shown in Fig.~\ref{fig:mse}, Fig.~\ref{fig:example}. As can be seen, for the same number of decoding steps, the \texttt{lowfreq}/\texttt{highfreq} trajectories achieve substantially lower denoising error than the \texttt{broadband} trajectories. Moreover, the error on \texttt{lowfreq}/\texttt{highfreq} nearly saturates after just two reverse steps, whereas \texttt{broadband} requires many more decoding steps to reach comparable performance. We further observe that, under few-step decoding, most of the remaining error for \texttt{lowfreq} lies in the high-frequency band. In practice, the robot’s low-level controller further filters out this high-frequency error, making two-step decoding nearly indistinguishable from using more reverse steps. See App.~\ref{app:toy} for details.

\begin{table*}[t]
\centering
\caption{Evaluation on the Robotwin2.0 benchmark. Each task is tested across 100 randomly generated scenes using 100 different seeds.}
\label{tab:robotwin2}

\small
\setlength{\tabcolsep}{4.0pt}
\renewcommand{\arraystretch}{1.15}
\newcommand{\task}[2]{\shortstack{\scriptsize #1\\\scriptsize #2}}

\begin{tabularx}{\textwidth}{@{}l *{8}{>{\centering\arraybackslash}X}@{}}
\toprule
\textbf{Method} & \task{Open}{Microwave} & \task{Hanging}{Mug} & \task{Move Pillbottle}{Pad} & \task{Stamp}{Seal} & \task{Beat Block}{Hammer} & \task{Place Phone}{Stand} & \ldots &
\task{Average}{(50 tasks)} \\
\midrule
DP & 5.0 & 8.0 & 1.0 & 2.0 & 42.0 & 13.0 & \ldots & 28.0 \\
DP3 & \underline{61.0} & \underline{17.0} & \underline{41.0} & 18.0 & 72.0 & 44.0 & \ldots & \underline{55.2} \\
Flow Policy & 17.0 & 6.0 & 14.0 & 20.0 & \underline{75.0} & 44.0 & \ldots & 41.1 \\
MP1 & 35.0 & 11.0 & 35.0 & \underline{28.0} & 68.0 & \underline{52.0} & \ldots & \underline{55.2} \\
\rowcolor{oursblue}
HDP3 \textbf{(Ours)} & \textbf{92.0}\,{\scriptsize(+31.0)} & \textbf{35.0}\,{\scriptsize(+18.0)} & \textbf{56.0}\,{\scriptsize(+15.0)} & \textbf{43.0}\,{\scriptsize(+15.0)} &
\textbf{88.0}\,{\scriptsize(+13.0)} & \textbf{65.0}\,{\scriptsize(+13.0)} & \ldots & \textbf{63.2}\,{\scriptsize(+8.0)} \\
\bottomrule
\end{tabularx}
\end{table*}

\begin{table*}[t]
\centering
\caption{Evaluation on the Adroit and MetaWorld benchmark. Each task is trained and tested across 3 different seeds. Unavailable results are denoted by ``--''.}
\label{tab:am}

\footnotesize
\setlength{\tabcolsep}{4pt}
\renewcommand{\arraystretch}{1.12}

\resizebox{\textwidth}{!}{%
\begin{tabular}{@{}l*{11}{c}@{}}
\toprule
\textbf{Method}
& \textbf{Hammer} & \textbf{Door} & \textbf{Pen}
& \textbf{Assembly} & \textbf{Disassemble} & \textbf{Hand-Insert}
& \textbf{Pick-Place-Wall} & \textbf{Push} & \textbf{Reach-Wall} & \textbf{Stick-Push}
& \textbf{Avg.} \\
\midrule

BCRNN
& 0.0 & 0.0 & 9.0
& 3.0 & 32.0 & --
& -- & -- & -- & --
& 8.8 \\
IBC
& 0.0 & 0.0 & 9.0
& 0.0 & 1.0 & --
& -- & -- & -- & --
& 2.0 \\
DP
& 45.0 & 37.0 & 13.0
& 15.0 & 43.0 & 9.0
& 5.0 & 11.0 & 59.0 & \underline{63.0}
& 30.0 \\

DP3
& \textbf{100.0} & \underline{62.0} & 43.7
& \underline{99.6} & 75.0 & \underline{25.3}
& 82.7 & 71.3 & 70.7 & \textbf{100.0}
& 73.0 \\

Flow Policy
& \textbf{100.0} & \textbf{65.3} & \underline{48.2}
& 92.7 & 60.7 & 18.6
& \underline{88.3} & \underline{79.7} & 72.7 & \textbf{100.0}
& 72.6 \\

MP1
& \textbf{100.0} & 56.5 & 44.0
& \textbf{100.0} & \textbf{93.7} & 19.3
& \textbf{91.3} & 79.3 & \underline{82.7} & \textbf{100.0}
& \underline{76.7} \\

\rowcolor{oursblue}
HDP3 \textbf{(Ours)}
& \textbf{100.0} & 57.3 & \textbf{48.7}
& \textbf{100.0} & \underline{87.7} & \textbf{33.3}
& \underline{88.7} & \textbf{83.0} & \textbf{85.0} & \textbf{100.0}
& \textbf{78.4} \\
\bottomrule
\end{tabular}
}

\vspace{-0.5mm}
\end{table*}

\vspace{-7pt}
\section{Simulation Experiments} \label{sec:exp}
\subsection{Experimental Setup}
\textbf{Simulation Benchmarks.}
We evaluate our method on three widely used benchmarks: RoboTwin2.0\citep{chen2025robotwin}, Adroit\citep{rajeswaran2017learning}, and MetaWorld\citep{yu2020meta}. RoboTwin2.0 focuses on dual-arm manipulation with diverse assets and scripted data generation. Adroit targets high-dimensional dexterous control for precise, long-horizon skills, while MetaWorld offers tiered single-arm tasks\citep{seo2023masked}. Collectively, these benchmarks cover both single- and dual-arm settings across varying dimensionalities, ensuring a comprehensive evaluation.

\textbf{Training and Evaluation Details.}
For a fair comparison, we follow the official training and evaluation protocols for all benchmarks. 
\begin{itemize}
    \item For RoboTwin2.0, we train the models using 50 expert demonstrations. We evaluate each task on 100 randomly generated scenes and report the success rate over these 100 evaluation episodes.

    \item  For Adroit and MetaWorld, we use 10 expert demonstrations to assess the data efficiency of the methods. We run three independent trials with random seeds ${0,1,2}$. During training, we evaluate the policy every 200 epochs; each evaluation consists of 20 rollouts per task, from which we compute the success rate. For each seed, we track the success rate over training and define ${\rm SR}_5$ as the average of the top five success rates; we report the average ${\rm SR}_5$ across the three seeds for all methods.
\end{itemize}

We adopt a DDIM\citep{song2020denoising} noise scheduler with 100 diffusion steps during training and 2 steps at inference, and optimize using AdamW\citep{loshchilov2017decoupled} with an initial learning rate of $1\times10^{-4}$ and a cosine decay schedule. The hidden dimension and depth of the DiM block are set to 128 and 6, respectively. To stabilize training, both actions and robot states are normalized to $[-1,1]$. All models are trained for 3{,}000 epochs on a single NVIDIA RTX 5880 Ada GPU, using a batch size of 256 for RoboTwin 2.0 and 128 for Adroit and MetaWorld. Detailed hyperparameter settings are provided in Appendix~\ref{app:impl}.

\textbf{Baselines.}
On RoboTwin2.0, we compare our method with the sota diffusion-based approaches, namely DP3\citep{ze20243d}, DP\citep{chi2025diffusion}, Flow Policy\citep{zhang2025flowpolicy} and MP1\citep{sheng2026mp1}.
On Adroit and MetaWorld, we include all the aforementioned diffusion-based methods, and further compare against two commonly used imitation-learning baselines, IBC\cite{florence2022implicit} and BCRNN\cite{mandlekar2021matters}.

\subsection{Comparison with State-of-the-art Methods}
\begin{wraptable}{r}{0.48\textwidth}
\vspace{-4mm}
\centering
\caption{\textbf{Model size and inference latency comparison.} All evaluations are conducted on a single RTX~5880 Ada GPU with the batch size set to 1.}
\label{tab:efficiency}
\footnotesize
\setlength{\tabcolsep}{4pt}
\renewcommand{\arraystretch}{1.1}
\begin{tabularx}{0.48\textwidth}{@{}Xccc@{}}
\toprule
\textbf{Method} & \textbf{Params (M)$\downarrow$} & \textbf{NFE} & \textbf{Latency (ms)$\downarrow$} \\
\midrule
DP & 85.6 & 100 & 460 \\
DP3 & 255.8 & 10 & 51.4 \\
Flow Policy & 255.8 & \textbf{1} & 7.04 \\
MP1 & 255.8 & \textbf{1} & 7.02 \\
HDP3 & \textbf{2.52} & 2 & \textbf{4.50} \\
\bottomrule
\end{tabularx}
\vspace{-3mm}
\end{wraptable}
\textbf{Model Size and Inference Latency.}
We measure the model size and inference latency of several diffusion-based policies; the results are reported in Tab. \ref{tab:efficiency}. Thanks to two-step inference and a compact model size, our HDP3 substantially reduces inference latency compared to the baseline method, DP3, while maintaining strong performance (as shown below). Moreover, while Flow Policy and MP1 achieve impressive latency via 1-step inference, it retains a massive parameter count ($\sim$256M). HDP3 achieves comparable low latency ($\sim$4.50ms) but with \textbf{two orders of magnitude fewer parameters}, significantly lowering the memory footprint for on-device deployment.

\textbf{RoboTwin 2.0.}
As shown in Tab.~\ref{tab:robotwin2}, HDP3 achieves the best overall performance on RoboTwin2.0, reaching an average success rate of \textbf{63.2\%} over 50 tasks. With only \textbf{two inference steps}, HDP3 outperforms the previous best baseline, DP3 (\underline{55.2\%}) and MP1 (\underline{55.2\%}), by \textbf{8.0 percentage points}, while also surpassing Flow Policy (40.9\%). The six tasks shown in Tab.~\ref{tab:robotwin2} correspond to the largest margins over the runner-up method: \textit{Open Microwave}(+31.0), \textit{Hanging Mug}(+18.0), \textit{Move Pillbottle Pad} (+15.0), \textit{Stamp Seal}(+15.0), \textit{Beat Block Hammer}(+13.0), and \textit{Place Phone Stand}(+13.0). These results highlight the
consistent advantage of HDP3 across diverse manipulation tasks.

\textbf{Adroit and MetaWorld.}
Tab.~\ref{tab:am} compares HDP3 with prior methods across 10 manipulation tasks on Adroit and MetaWorld. HDP3 achieves the best overall performance with \textbf{78.4\%} average success, outperforming the strongest baseline MP1 (\underline{76.7\%}) by \textbf{1.7 percentage points}, and exceeding DP3 (73.0\%) and Flow Policy (72.6\%) by 5.4 and 5.8 percentage points, respectively. HDP3 attains the best performance on \textit{Pen}, \textit{Hand-Insert}, \textit{Push}, and \textit{Reach-Wall}, while matching the best result on \textit{Hammer}, \textit{Assembly}, and \textit{Stick-Push}. Despite using an extremely small number of parameters and only two inference steps, HDP3 maintains---and in some cases surpasses---the performance of prior state-of-the-art methods.

\subsection{Ablation Studies} \label{sec:aba}
\begin{table}[h]
\centering
\caption{\textbf{Ablation on design choices in HDP3.} To ensure a fair comparison, we tune the hidden size for each setting to keep the number of parameters comparable.}
\label{tab:ablation}
\resizebox{\columnwidth}{!}{%
\begin{tabular}{l|c|cc|cc|cc|cc}
\toprule
\textbf{Method} & \textbf{Num.Params.(M)} &
\multicolumn{2}{c|}{\textbf{Door}} &
\multicolumn{2}{c|}{\textbf{Hammer}} &
\multicolumn{2}{c|}{\textbf{Pen}} &
\multicolumn{2}{c}{\textbf{Avg.}} \\
\cmidrule(lr){3-4}\cmidrule(lr){5-6}\cmidrule(lr){7-8}\cmidrule(lr){9-10}
 &  & \textbf{${\rm SR}_5$} & Loss & \textbf{${\rm SR}_5$} & Loss & \textbf{${\rm SR}_5$} & Loss & \textbf{${\rm SR}_5$} & Loss \\
\midrule
HDP3 & 2.52 &
\textbf{54.3} & 6e-4 &
\textbf{100} & 2e-4 &
\textbf{48.7} & 5e-4 &
\textbf{67.7} & 4.3e-4 \\
Vanilla-MLP    & 2.52 &
0.0 & 3e-2 &
0.0 & 1e-2 &
0.0 & 3e-2 &
0.0 & 2.3e-2 \\
Vanilla-UNet   & 2.64 &
25.7 & 3e-3 &
93.3 & 3e-4 &
42.0 & 1e-3 &
53.7 & 1.4e-3 \\
\bottomrule
\end{tabular}}
\vspace{-0.2cm}
\end{table}

\textbf{Necessity of the Mixer-based Architecture.}
To validate the design of HDP3, we conduct an ablation study that compares different architectures under the same parameter budget. The results are summarized in Tab.~\ref{tab:ablation}. Here, Vanilla-MLP denotes a pure MLP architecture with FiLM-based conditioning, while Vanilla-UNet is obtained by reducing the intermediate channel width in the decoder of the original U-Net. We observe that the pure MLP baseline achieves an average ${\rm SR}_5$ of $0$ across all three tasks. Specifically, the training loss for Vanilla-MLP converged extremely poorly, indicating that a plain MLP is far from sufficient for robotic manipulation. Notably, our DiM block---which augments the MLP with residual connections and a temporal fusion module (incurring almost no additional parameters)---substantially improves performance. 
Furthermore, the U-Net variant with reduced channel width suffers a significant performance drop compared to the original DP3 and is substantially outperformed by our HDP3. These results suggest that, under a strict parameter budget, the DiM block offers superior parameter efficiency and information fusion capabilities compared to standard U-Net architectures.

\textbf{Impact of the Number of Function Evaluations (NFE).}
We further demonstrate that two-step decoding is sufficient for robot trajectory generation through rollout success-rate evaluations on three Adroit tasks: \textit{Door}, \textit{Hammer}, and \textit{Pen}. Specifically, we evaluate the final checkpoint after 3000 training epochs. For each NFE setting, we run 5{,}000 evaluation episodes to reduce statistical variance and report the mean score. The results are shown in Tab.~\ref{tab:ablation_nfe}. As can be seen, using only one inference step leads to a noticeable performance drop compared to multi-step settings. However, two-step inference already achieves strong performance, and additional steps (e.g., 5 or 10) do not yield further gains. This indicates that multi-step decoding is not necessary, and that two-step decoding is already sufficient for closed-loop execution.


\vspace{-7pt}
\begin{table}[h]
\centering
\caption{\textbf{Ablation on NFE.} We report the average success rate (\%) over 5,000 randomly generated scenes for each task using the final checkpoint (3,000 epochs). \textit{NFE} denotes the Number of
Function Evaluations.}
\label{tab:ablation_nfe}
\small
\renewcommand{\arraystretch}{1.05}
\begin{tabular*}{\columnwidth}{@{\extracolsep{\fill}}lcccccccc@{}}
\toprule
& \multicolumn{4}{c}{\textbf{HDP3}} & \multicolumn{4}{c}{\textbf{DP3}} \\
\cmidrule(lr){2-5} \cmidrule(lr){6-9}
\textbf{NFE} & \textbf{Door} & \textbf{Hammer} & \textbf{Pen} & \textbf{Avg.}
           & \textbf{Door} & \textbf{Hammer} & \textbf{Pen} & \textbf{Avg.} \\
\midrule
1  & 34.2 & 91.5  & 33.4 & 53.0 & 44.6 & 100.0 & 27.8 & 57.5 \\
2  & 37.4 & 100.0 & 37.2 & 58.2 & 47.4 & 100.0 & 30.5 & 59.3 \\
5  & 37.7 & 100.0 & 35.5 & 57.7 & 46.1 & 100.0 & 31.7 & 59.3 \\
10 & 36.8 & 100.0 & 37.3 & 58.0 & 47.8 & 100.0 & 31.8 & 59.9 \\
\bottomrule
\end{tabular*}
\vspace{-3.0mm}
\end{table}


\subsection{Real-world Experiments}
\textbf{Experiment Settings.}
We validate the effectiveness of our method on an AgileX Piper robot. Real-world visual observations are captured with a single Intel RealSense D455 camera mounted globally. The model is deployed on an NVIDIA RTX 4060 GPU for on-board action inference. We collect expert demonstrations via teleoperation using a leader-follower setup, yielding 50 trajectories for training. The action space is defined as 6-DoF joint positions. See App.~\ref{app:impl} for more implementation details.

\textbf{Results.}
We report the success rates across 15 real-world trials in Tab. \ref{tab:real_sr}. Consistent with our simulation findings, HDP3-base exhibits strong sim-to-real transferability. It outperforms the DP3 baseline by 15.3\% on average, validating that the efficiency gains from the DiM decoder do not come at the cost of generalization capability, even under the constraints of real-world noise.




\vspace{-7pt}
\begin{table}[h]

    \centering
    \caption{\textbf{Real-world Experiments Success Rate.}}
    \label{tab:real_sr}
    \renewcommand{\arraystretch}{1.15}
    \setlength{\tabcolsep}{8pt}
    \begin{tabular}{lcc}
        \toprule
        Task &  DP3 & HDP3 \textbf{(Ours)} \\
        \midrule
        Place Object &  53.3 & \textbf{73.3} (\textbf{$\uparrow$20.0}) \\
        Adjust Bottle &  33.3 & \textbf{46.7} (\textbf{$\uparrow$13.4}) \\
        Stack Blocks Two &  6.7 & \textbf{20} (\textbf{$\uparrow$13.3}) \\
        \bottomrule
    \end{tabular}
    \vspace{-3.0mm}
\end{table}

\section{Conclusion}
In this work, we introduced \textbf{HDP3}, a compact yet powerful 3D visuomotor policy. Motivated by the insight that robot trajectories are strongly dominated by low-frequency components in the frequency domain, we provide a theoretical analysis showing that two-step decoding does not affect closed-loop execution, and further replace the parameter-heavy conditional decoder with a lightweight MLP-Mixer-based DiM block. Synthetic-data experiments validate the theory-driven intuition, while simulation results show that HDP3 achieves state-of-the-art performance with less than 1\% of the parameters of prior methods. Real-world experiments further demonstrate the practical applicability of the proposed method. For limitations, please refer to App.~\ref{app:lim}.

\newpage
\bibliographystyle{plain}
\bibliography{nips}

\newpage
\appendix

\section{Proof of Theorem~\ref{th:1}} \label{app:proof}
We first derive the closed-form expression of the optimal estimator in the frequency domain, given in the following lemma.
\begin{lemma} \label{lemma:1}
     Suppose that $y_0 \sim \mathcal{N}(\mathbf{0}, \mathbf{\Sigma})$ and $y_t = \alpha_t y_0 + \sigma_t \xi$, where $\xi \sim \mathcal{N}(\mathbf{0}, \mathbf{I})$. Then, the posterior distribution is Gaussian, with mean and covariance given by:
     \begin{gather}
         \mathbb{E}[y_0|y_t]=\alpha_t {\bf \Sigma}(\alpha_t^2 {\bf \Sigma} + \sigma_t^2 {\bf I})^{-1}y_t \\
         \operatorname{Cov}[y_0|y_t]= \sigma_t^2 {\bf \Sigma} (\alpha_t^2 {\bf \Sigma} + \sigma_t^2 {\bf I})^{-1}
     \end{gather}
\end{lemma}
\begin{proof}
     Since $y_0$ and $\xi$ are independent, the pair $(y_0, \xi)$ is jointly Gaussian. Moreover, since we have:
     \begin{equation}
         \begin{bmatrix}
             y_0 \\
             y_t
         \end{bmatrix} = \begin{bmatrix}
             {\bf I} & {\bf 0} \\
             \alpha_t {\bf I} & \sigma_t {\bf I}
         \end{bmatrix}\begin{bmatrix}
             y_0 \\
             \xi
         \end{bmatrix}
     \end{equation}
     Therefore, $(y_0, y_t)$ is also jointly Gaussian and satisfies:
     \begin{gather}
         (y_0, y_t) \sim \mathcal{N}({\bf 0}, {\bf C}) \\
         {\bf C} = \begin{bmatrix}
             {\bf I} & \alpha_t {\bf \Sigma} \\
             \alpha_t {\bf \Sigma} & \alpha_t^2 {\bf \Sigma}+ \sigma
             _t^2 {\bf I}
         \end{bmatrix}
     \end{gather}
     Hence, the posterior distribution is still Gaussian, and its mean and covariance are given by:
     \begin{equation}
         \mathbb{E}[y_0|y_t]=\operatorname{Cov}(y_0,y_t)\operatorname{Cov}(y_t,y_t)^{-1}y_t = \alpha_t {\bf \Sigma}(\alpha_t^2 {\bf \Sigma}+ \sigma_t^2 {\bf I})^{-1}y_t 
     \end{equation}
     Meanwhile,
     \begin{equation}
         \operatorname{Cov}[y_0|y_t]=\operatorname{Cov}(y_0,y_0)-\operatorname{Cov}(y_0,y_t)\operatorname{Cov}(y_t,y_t)^{-1}\operatorname{Cov}(y_t,y_0)
     \end{equation}
     By setting ${\bf A} = \alpha_t^2 {\bf \Sigma}+ \sigma_t^2 {\bf I}$, we obtain:
     \begin{align}
         \operatorname{Cov}[y_0|y_t]&={\bf \Sigma}-\alpha_t^2 {\bf \Sigma} {\bf A}^{-1}{\bf \Sigma} \\
         &= {\bf \Sigma}{\bf A}^{-1} ({\bf A}-\alpha_t^2{\bf \Sigma}) \\
         &= \sigma_t^2 {\bf \Sigma}{\bf A}^{-1} \\
         &= \sigma_t^2 {\bf \Sigma}(\alpha_t^2 {\bf \Sigma}+ \sigma_t^2 {\bf I})^{-1}
     \end{align}
\end{proof}
We next derive bounds on the posterior covariance, as stated in the following lemma.
\begin{lemma} \label{lemma:2}
    Under the Loewner partial order, the posterior covariance satisfies:
    \begin{gather}
        {\bf 0} \preceq  \operatorname{Cov}[y_0|y_t] \preceq {\bf \Sigma} \\
        {\bf 0} \preceq  \operatorname{Cov}[y_0|y_t] \preceq \frac{\sigma_t^2}{\alpha_t^2}{\bf I}
    \end{gather}
\end{lemma}
\begin{proof}
    Let the spectral decomposition of ${\bf \Sigma}$ be
    \begin{equation}
        {\bf \Sigma} = {\bf Q}\operatorname{diag}(\nu_1, \dots, \nu_n){\bf Q}^{\top}
    \end{equation}
    where ${\bf Q}$ is an orthogonal matrix, and $\nu_i$ denotes the nonnegative eigenvalues. Setting ${\bf A} = \alpha_t^2 {\bf \Sigma}+ \sigma_t^2 {\bf I}$, we obtain:
    \begin{align}
        {\bf A} &= \alpha_t^2 {\bf \Sigma}+ \sigma_t^2 {\bf I} \\
        &= {\bf Q}\operatorname{diag}(\alpha_t^2\nu_1, \dots, \alpha_t^2\nu_n){\bf Q}^{\top} + {\bf Q}\operatorname{diag}(\sigma_t^2, \dots, \sigma_t^2){\bf Q}^{\top} \\
        &= {\bf Q}\operatorname{diag}(\alpha_t^2\nu_1+\sigma_t^2, \dots, \alpha_t^2\nu_n+\sigma_t^2){\bf Q}^{\top}
    \end{align}
    Thus,
    \begin{equation}
        {\bf A}^{-1} = {\bf Q}\operatorname{diag}(\frac{1}{\alpha_t^2\nu_1+\sigma_t^2}, \dots, \frac{1}{\alpha_t^2\nu_1+\sigma_t^2}){\bf Q}^{\top}
    \end{equation}
    Furthermore, we have:
    \begin{align}
        \operatorname{Cov}[y_0|y_t] &= \sigma_t^2 {\bf \Sigma} {\bf A}^{-1} \\
        &= {\bf Q}\operatorname{diag}(\sigma_t^2\nu_1, \dots, \sigma_t^2\nu_n){\bf Q}^{\top}{\bf A}^{-1}\\
        &= {\bf Q}\operatorname{diag}(\frac{\sigma_t^2\nu_1}{\alpha_t^2\nu_1+\sigma_t^2}, \dots, \frac{\sigma_t^2\nu_n}{\alpha_t^2\nu_n+\sigma_t^2}){\bf Q}^{\top}
    \end{align}
    Consider the scalar function $f(\lambda) = \frac{\sigma_t^2\lambda}{\alpha_t^2\lambda+\sigma_t^2}$. Clearly, for $\lambda \ge 0$, we have:
    \begin{gather}
        0 \leq f(\lambda) \leq \lambda \\
        0 \leq f(\lambda) \leq \frac{\sigma_t^2}{\alpha_t^2}
    \end{gather}
    Thus, we obtain:
    \begin{gather}
        {\bf 0} \preceq  \operatorname{Cov}[y_0|y_t] \preceq {\bf Q}\operatorname{diag}(\nu_1, \dots, \nu_n){\bf Q}^{\top} = {\bf \Sigma} \\
        {\bf 0} \preceq  \operatorname{Cov}[y_0|y_t] \preceq {\bf Q}\operatorname{diag}(\frac{\sigma_t^2}{\alpha_t^2}, \dots, \frac{\sigma_t^2}{\alpha_t^2}){\bf Q}^{\top}  = \frac{\sigma_t^2}{\alpha_t^2}{\bf I}
    \end{gather}
\end{proof}
Finally, we derive the upper bound on the estimation error of the optimal estimator as follows.
\begin{theorem}[Frequency-Domain Error Bound of the Optimal Estimator]
    Under the assumptions in Sec.~\ref{sec:2s},, the optimal estimation errors in the low- and high-frequency components satisfy, respectively:
    \begin{gather}
        e_L = \frac{1}{n}\mathbb{E}\|\mathbf{P}_L\left(y_0-\mathbb{E}[y_0|y_t]\right)\|^2 \leq \frac{m}{n}\frac{\sigma_t^2}{\alpha_t^2} \\
        e_H = \frac{1}{n}\mathbb{E}\|\mathbf{P}_H\left(y_0-\mathbb{E}[y_0|y_t]\right)\|^2 \leq \frac{\eta}{n} \operatorname{Tr}(\mathbf{\Sigma})
    \end{gather}
     where $n$ is the number of frames in the predicted trajectory, $e_L$ and $e_H$ denote the mean squared errors (MSEs) of the low- and high-frequency components, respectively. Consequently, the total error satisfies:
     \begin{equation}
         e = \frac{1}{n}\mathbb{E}\|x_0-\mathbb{E}[x_0|x_t]\|^2 = \frac{1}{n}\mathbb{E}\|y_0-\mathbb{E}[y_0|y_t]\|^2 = e_L +e_H \leq \frac{m}{n}\frac{\sigma_t^2}{\alpha_t^2} + \frac{\eta}{n} \operatorname{Tr}(\mathbf{\Sigma})
     \end{equation}
     The relative error is bounded by:
     \begin{equation}
         \hat{e} = \frac{e}{\operatorname{Tr}({\bf \Sigma})} \leq \frac{m}{n\operatorname{Tr}({\bf \Sigma})}\frac{\sigma_t^2}{\alpha_t^2} + \frac{\eta}{n}
     \end{equation}
\end{theorem}
\begin{proof}
    We first consider the low-frequency error $e_L$:
    \begin{align}
        e_L &= \frac{1}{n}\mathbb{E}\|\mathbf{P}_L\left(y_0-\mathbb{E}[y_0|y_t]\right)\|^2 \\ 
        &= \frac{1}{n}\mathbb{E}\left[\left(y_0-\mathbb{E}[y_0|y_t]\right)^{\top}\mathbf{P}_L\left(y_0-\mathbb{E}[y_0|y_t]\right)\right] \\
        &= \frac{1}{n}\operatorname{Tr}\left(\mathbf{P}_L\operatorname{Cov}[y_0|y_t]\right) \\
        & \leq \frac{1}{n}\frac{\sigma_t^2}{\alpha_t^2}\operatorname{Tr}\left(\mathbf{P}_L{\bf I}\right) \\
        &= \frac{m}{n} \frac{\sigma_t^2}{\alpha_t^2}
    \end{align}
    We then turn to the high-frequency error:
    \begin{align}
        e_H &= \frac{1}{n}\mathbb{E}\|\mathbf{P}_H\left(y_0-\mathbb{E}[y_0|y_t]\right)\|^2 \\ 
        &= \frac{1}{n}\mathbb{E}\left[\left(y_0-\mathbb{E}[y_0|y_t]\right)^{\top}\mathbf{P}_H\left(y_0-\mathbb{E}[y_0|y_t]\right)\right] \\
        &= \frac{1}{n}\operatorname{Tr}\left(\mathbf{P}_H\operatorname{Cov}[y_0|y_t]\right) \\
        & \leq \frac{1}{n}\operatorname{Tr}\left(\mathbf{P}_H{\bf \Sigma}\right) \\
        & \leq \frac{\eta}{n} \operatorname{Tr}\left({\bf \Sigma}\right)
    \end{align}
    Furthermore, the total error satisfies:
    \begin{align}
        e &= \frac{1}{n}\mathbb{E}\|y_0-\mathbb{E}[y_0|y_t]\|^2 \\
        &= \frac{1}{n}\mathbb{E}\left[\left(y_0-\mathbb{E}[y_0|y_t]\right)^{\top}\mathbf{I}\left(y_0-\mathbb{E}[y_0|y_t]\right)\right] \\
        &= \frac{1}{n}\mathbb{E}\left[\left(y_0-\mathbb{E}[y_0|y_t]\right)^{\top}\mathbf{P}_L\left(y_0-\mathbb{E}[y_0|y_t]\right)\right] + \frac{1}{n}\mathbb{E}\left[\left(y_0-\mathbb{E}[y_0|y_t]\right)^{\top}\mathbf{P}_H\left(y_0-\mathbb{E}[y_0|y_t]\right)\right] \\
        &= \frac{1}{n}\mathbb{E}\|\mathbf{P}_L\left(y_0-\mathbb{E}[y_0|y_t]\right)\|^2 + \frac{1}{n}\mathbb{E}\|\mathbf{P}_H\left(y_0-\mathbb{E}[y_0|y_t]\right)\|^2 \\
        &= e_L + e_H \\
        & \leq \frac{m}{n} \frac{\sigma_t^2}{\alpha_t^2} + \frac{\eta}{n} \operatorname{Tr}\left({\bf \Sigma}\right)
    \end{align}
\end{proof}
\begin{remark}
    The above discussion can be extended to the more general case where the data are concentrated in any frequency band of the DCT spectrum.
\end{remark}

\section{Details of the Synthetic Data Experiments} \label{app:toy}

\begin{figure}[h]
  \centering
  \includegraphics[width=\linewidth]{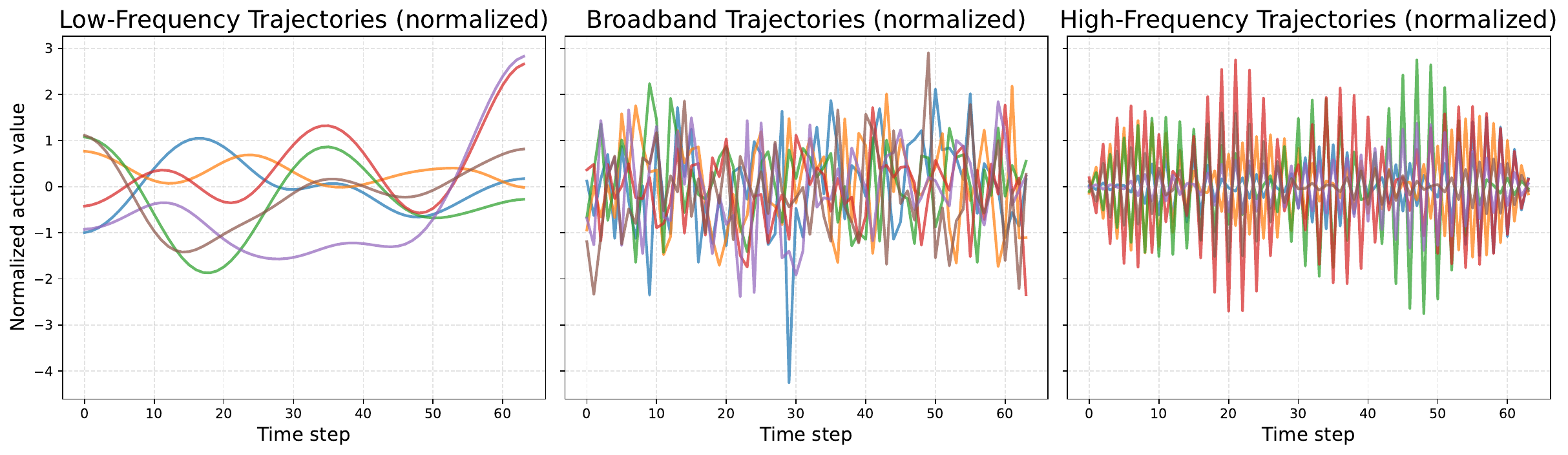}
  \caption{Examples of normalized synthetic trajectories from the low-frequency-dominant, broadband, and high-frequency-dominant data generators. For clarity, we plot the first action dimension only. Even in the standardized training space, \texttt{lowfreq} trajectories remain substantially smoother in the time domain, \texttt{broadband} trajectories exhibit richer multi-band variations, and \texttt{highfreq} trajectories oscillate rapidly across time.}
  \label{fig:toy_dataset_examples}
\end{figure}

\subsection{Data Synthesis}

We construct three synthetic trajectory datasets with controlled spectral concentration. For each sample, we first draw DCT-domain coefficients
\[
z_{i,d,k} \sim \mathcal{N}(0,\lambda_k),
\]
where $i$ indexes the trajectory, $d$ indexes the action dimension, and $k$ indexes the DCT mode. We then transform the coefficients back to the time domain through the inverse DCT,
\[
x_{i,d} = U^\top z_{i,d},
\]
where $U$ denotes the orthogonal DCT matrix. In this way, the training and sampling procedures are performed entirely in the time domain, while the DCT domain is used only to construct data with controlled spectra and to analyze the resulting errors.

We consider three data generators:
\begin{itemize}
\item \texttt{lowfreq}: the first six DCT modes are fixed as $(0.25, 1.00, 0.90, 0.65, 0.35, 0.20)$, while the remaining modes receive an exponentially decaying tail whose total energy is fixed to \(1\%\) of the full spectrum;
\item \texttt{broadband}: \(\lambda_k = 1\);
\item \texttt{highfreq}: the last six DCT modes are fixed as $(0.25, 1.00, 0.90, 0.65, 0.35, 0.20)$, while the preceding modes receive an exponentially decaying prefix whose total energy is fixed to \(1\%\) of the full spectrum.
\end{itemize}
Under this construction, the first 6 modes account for 99\% of the total energy in \texttt{lowfreq}, the last 6 modes account for 99\% of the total energy in \texttt{highfreq}, and energy is uniformly distributed in \texttt{broadband}. For \texttt{lowfreq} and \texttt{broadband}, we define the 95\% spectral cutoff by cumulative energy from low to high, which yields the first 6 modes for \texttt{lowfreq} and the first 61 modes for \texttt{broadband}. For \texttt{highfreq}, we instead define the dominant high-frequency band as the smallest suffix of DCT modes whose cumulative energy from high to low reaches 95\%, which yields the last 6 modes. This construction provides a controlled contrast among strongly low-frequency-dominant, spectrally broad, and strongly high-frequency-dominant trajectories while keeping all other factors fixed. After synthesis, we normalize each dataset in the time domain on a per-action-dimension basis using the training-set mean and standard deviation computed over all samples and timesteps. All model training, sampling, and quantitative errors are computed in this standardized space. Fig.~\ref{fig:toy_dataset_examples} presents examples of the synthetic data and illustrates that the \texttt{lowfreq} trajectories are significantly smoother, whereas the \texttt{highfreq} trajectories exhibit rapid oscillations.

\subsection{Training Details}

Each dataset contains 10,000 training trajectories and 256 test trajectories. All trajectories have length \(n=64\) and action dimension \(d=4\). We train a diffusion model directly in the time domain with the standard DDPM\citep{ho2020denoising} \(\epsilon\)-prediction objective, using the mean squared error loss.

The denoiser is a small 1D convolutional residual network. It uses a \(1\times1\) input projection, 8 residual blocks with dilations \(1,2,4,8\) repeated cyclically, 64 hidden channels, sinusoidal timestep embeddings, and a timestep embedding dimension of 128. The model is optimized with Adam using a learning rate of \(10^{-3}\), a batch size of 256, and 60 training epochs.

For the forward diffusion process, we use 100 diffusion timesteps with a linear variance schedule from \(\beta_{\min}=10^{-4}\) to \(\beta_{\max}=2\times 10^{-2}\). At inference time, we use deterministic DDIM\citep{song2020denoising} sampling with \(\eta=0\). To isolate the effect of the number of reverse updates, all step ablations share the same initial terminal noise. Following our main text, we treat the 100-step DDIM sample as the pseudo-truth trajectory and compare shorter DDIM chains against this reference.

We evaluate reverse-step counts $K \in \{1,\dots,100\}$. We further decompose the action-space error into its low-band and high-band components according to the 95\% spectral cutoff of each dataset, in order to quantify how the low- and high-frequency errors evolve as the number of decoding steps changes. For \texttt{lowfreq} and \texttt{broadband}, the high-band is defined by the complement of the standard low-to-high 95\% cutoff. For \texttt{highfreq}, the high-band is defined directly as the smallest suffix of DCT modes whose cumulative energy from high to low reaches 95\%. In addition, we use the built-in controller in RoboTwin2.0\citep{chen2025robotwin} to execute the decoded trajectories and record the resulting tracking MSE, in order to examine whether the low-level controller can further suppress the impact of decoding errors.

\begin{figure}[htbp]
    \centering
    \includegraphics[width=1.0\linewidth]{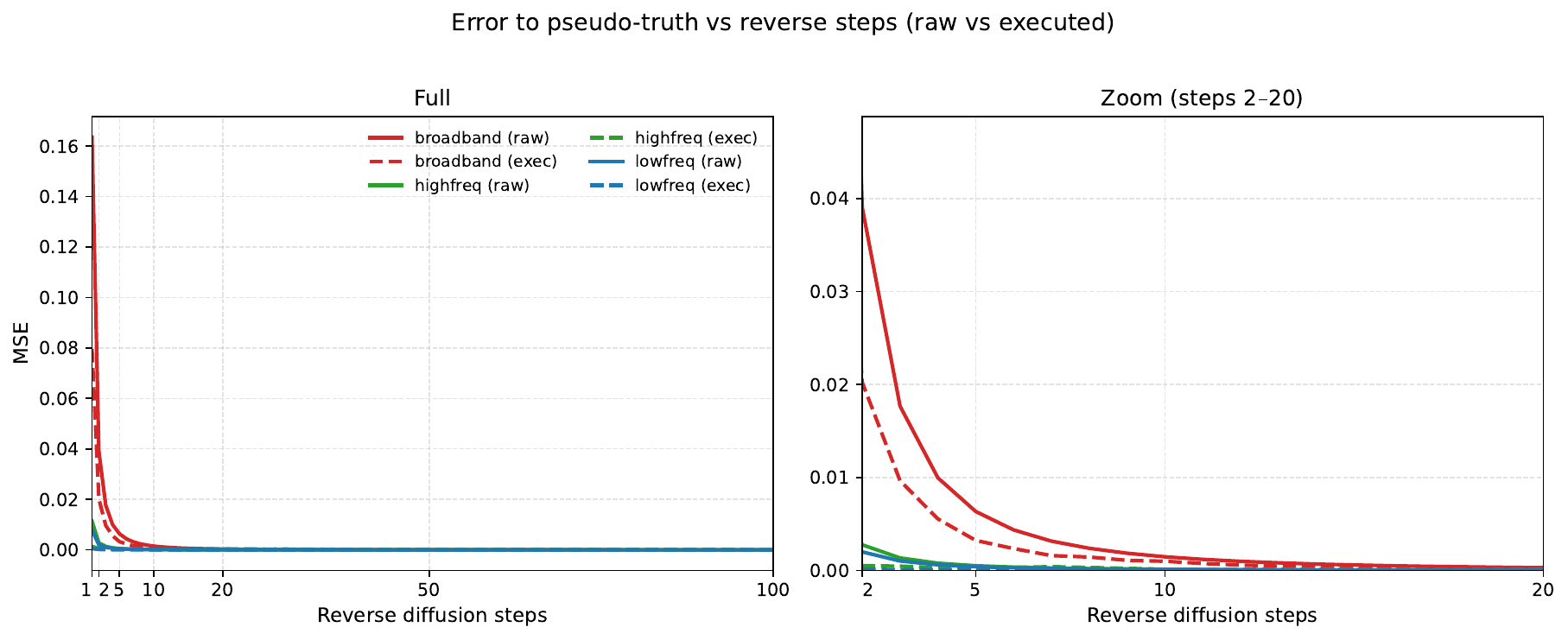}
    \caption{Decoding Error at Different Sampling Steps}
    \label{fig:exc}
    \vspace{-3.0mm}
\end{figure}

\begin{table}[htbp]
\centering
\footnotesize
\setlength{\tabcolsep}{3pt}
\resizebox{\linewidth}{!}{%
\begin{tabular}{lccccccc}
\hline
Dataset & Steps & Action MSE & Low-band MSE & High-band MSE & Exec MSE & Exec Low-band MSE & Exec High-band MSE \\
\hline
\texttt{lowfreq} & 1   & 0.007805 & \(1.333\times10^{-3}\) & 0.006472 & \(4.860\times10^{-4}\) & \(4.141\times10^{-4}\) & \(7.188\times10^{-5}\) \\
\texttt{lowfreq} & 2   & 0.001997 & \(3.598\times10^{-4}\) & 0.001637 & \(1.475\times10^{-4}\) & \(1.203\times10^{-4}\) & \(2.714\times10^{-5}\) \\
\texttt{lowfreq} & 5   & \(4.008\times10^{-4}\) & \(6.205\times10^{-5}\) & \(3.388\times10^{-4}\) & \(4.208\times10^{-5}\) & \(3.557\times10^{-5}\) & \(6.512\times10^{-6}\) \\
\texttt{lowfreq} & 10  & \(9.587\times10^{-5}\) & \(1.445\times10^{-5}\) & \(8.142\times10^{-5}\) & \(1.498\times10^{-5}\) & \(1.302\times10^{-5}\) & \(1.957\times10^{-6}\) \\
\texttt{lowfreq} & 50  & \(1.175\times10^{-6}\) & \(1.866\times10^{-7}\) & \(9.880\times10^{-7}\) & \(2.161\times10^{-7}\) & \(1.848\times10^{-7}\) & \(3.124\times10^{-8}\) \\
\texttt{lowfreq} & 100 & 0 & 0 & 0 & 0 & 0 & 0 \\
\hline
\texttt{broadband} & 1   & 0.1635 & 0.1562 & \(7.382\times10^{-3}\) & 0.07955 & 0.07955 & \(6.317\times10^{-8}\) \\
\texttt{broadband} & 2   & 0.03905 & 0.03729 & \(1.760\times10^{-3}\) & 0.02022 & 0.02022 & \(7.184\times10^{-8}\) \\
\texttt{broadband} & 5   & \(6.339\times10^{-3}\) & \(6.054\times10^{-3}\) & \(2.852\times10^{-4}\) & \(3.228\times10^{-3}\) & \(3.228\times10^{-3}\) & \(5.083\times10^{-8}\) \\
\texttt{broadband} & 10  & \(1.467\times10^{-3}\) & \(1.401\times10^{-3}\) & \(6.592\times10^{-5}\) & \(9.902\times10^{-4}\) & \(9.901\times10^{-4}\) & \(3.150\times10^{-8}\) \\
\texttt{broadband} & 50  & \(1.855\times10^{-5}\) & \(1.771\times10^{-5}\) & \(8.330\times10^{-7}\) & \(9.496\times10^{-5}\) & \(9.495\times10^{-5}\) & \(1.308\times10^{-8}\) \\
\texttt{broadband} & 100 & 0 & 0 & 0 & 0 & 0 & 0 \\
\hline
\texttt{highfreq} & 1   & 0.011386 & 0.010280 & \(1.106\times10^{-3}\) & 0.001351 & 0.001351 & \(1.676\times10^{-8}\) \\
\texttt{highfreq} & 2   & 0.002760 & 0.002461 & \(2.987\times10^{-4}\) & \(5.171\times10^{-4}\) & \(5.170\times10^{-4}\) & \(1.811\times10^{-8}\) \\
\texttt{highfreq} & 5   & \(5.073\times10^{-4}\) & \(4.527\times10^{-4}\) & \(5.460\times10^{-5}\) & \(2.490\times10^{-4}\) & \(2.490\times10^{-4}\) & \(1.958\times10^{-8}\) \\
\texttt{highfreq} & 10  & \(1.152\times10^{-4}\) & \(1.022\times10^{-4}\) & \(1.294\times10^{-5}\) & \(1.253\times10^{-4}\) & \(1.253\times10^{-4}\) & \(1.627\times10^{-8}\) \\
\texttt{highfreq} & 50  & \(1.349\times10^{-6}\) & \(1.181\times10^{-6}\) & \(1.682\times10^{-7}\) & \(7.722\times10^{-6}\) & \(7.718\times10^{-6}\) & \(3.717\times10^{-9}\) \\
\texttt{highfreq} & 100 & 0 & 0 & 0 & 0 & 0 & 0 \\
\hline
\end{tabular}%
}
\caption{Few-step DDIM decoding errors measured against the 100-step pseudo-truth trajectory. We report both raw action-space errors and errors after RoboTwin's strict TOPP-based execution map. For \texttt{highfreq}, the high-band is defined as the smallest suffix of DCT modes whose cumulative energy from high to low reaches 95\%.}
\label{tab:toy_results}
\end{table}

\subsection{Experimental Results}

Table~\ref{tab:toy_results} and Figure~\ref{fig:exc} summarize the quantitative results for selected DDIM step counts. Both \texttt{lowfreq} and \texttt{highfreq} are spectrally concentrated in the DCT domain, and both therefore achieve substantially smaller reconstruction errors than \texttt{broadband}, whose energy is spread across all modes. At two reverse steps, the action-space MSE is \(1.997\times10^{-3}\) for \texttt{lowfreq} and \(2.760\times10^{-3}\) for \texttt{highfreq}, both far below the \texttt{broadband} value of \(3.905\times10^{-2}\). At the same time, \texttt{lowfreq} remains consistently easier than \texttt{highfreq}. This is consistent with the spectral bias of neural networks\citep{rahaman2019spectral}, namely that low-frequency components are fitted more easily than high-frequency ones. As a result, although both datasets are spectrally concentrated, concentration in the low-frequency subspace is more favorable for few-step denoising than concentration in rapidly oscillating modes.

A further difference appears in the structure of the residual error. On \texttt{lowfreq}, the error drops sharply within the first few reverse updates: the action-space MSE decreases from
\(7.805\times10^{-3}\) at one step to \(1.997\times10^{-3}\) at two steps, and further to \(4.008\times10^{-4}\) at five steps. Moreover, the residual few-step error is already concentrated in the high-frequency band: at two steps, the low-band MSE is only \(3.598\times10^{-4}\), whereas the high-band MSE remains \(1.637\times10^{-3}\). This gap becomes even more pronounced after applying RoboTwin's TOPP-based execution map. For \texttt{lowfreq}, the executed-space MSE at two steps is reduced from \(1.997\times10^{-3}\) in action space to \(1.475\times10^{-4}\), and the executed high-band MSE drops from \(1.637\times10^{-3}\) to only \(2.714\times10^{-5}\).

By contrast, \texttt{broadband} still requires substantially more reverse updates: even after the same execution map is applied, its two-step executed MSE remains \(2.022\times10^{-2}\). The \texttt{highfreq} control lies between the two extremes. Under the high-frequency-end 95\% cutoff, its two-step error decomposes into a low-band MSE of \(2.461\times10^{-3}\) and a high-band MSE of only \(2.987\times10^{-4}\), while the executed low-band and high-band MSEs are \(5.170\times10^{-4}\) and \(1.811\times10^{-8}\), respectively. Thus, even after isolating the dominant high-frequency suffix, the residual few-step error in \texttt{highfreq} is not concentrated in that band as strongly as the \texttt{lowfreq} residual is concentrated in its high-frequency tail. Overall, these observations show that spectral concentration alone is
not sufficient: few-step saturation depends specifically on concentration in low-frequency, task-relevant modes, while residual high-frequency errors are especially easy to suppress by the downstream low-pass execution pipeline.

\section{Implementation Details} \label{app:impl}
\subsection{Details of Simulation Experiment Settings}
In this section, we provide the essential hyperparameters required to reproduce our results. Detailed hyperparameter choices are listed in Tab. \ref{tab:hyperparameters}.
\begin{table}[h]
\centering
\caption{Hyperparameter Settings}
\label{tab:hyperparameters}
\setlength{\tabcolsep}{8pt}
\renewcommand{\arraystretch}{1.15}
\begin{tabular}{@{}llr@{}}
\toprule
\textbf{Category} & \textbf{Hyperparameter} & \textbf{Value} \\
\midrule
\multicolumn{3}{@{}l}{\textbf{Training}} \\
\midrule
& Batch size (Robotwin2.0) & 256 \\
& Batch size (Adroit \& MetaWorld) & 128 \\
& Num. epochs & 3000 \\
& Optimizer & AdamW \\
& Weight decay(Robotwin2.0) & $1\times10^{-6}$ \\
& Weight decay(Adroit \& MetaWorld)& $1\times10^{-8}$ \\
& LR scheduler & Cosine \\
& LR warmup steps & 500 \\
& Learning rate & $1\times10^{-4}$ \\
& Horizon (Robotwin2.0) & 8 \\
& Horizon (Adroit \& MetaWorld) & 16 \\
& Num. action steps (Robotwin2.0) & 6 \\
& Num. action steps (Adroit \& MetaWorld) & 8 \\
& Observation steps (Robotwin2.0) & 3 \\
& Observation steps (Adroit \& MetaWorld) & 2 \\
& Encoder output dim. & 64 \\
& Diffusion timestep dim. & 64 \\
& MLP expansion ratio. & 4 \\
\midrule
\multicolumn{3}{@{}l}{\textbf{Inference}} \\
\midrule
& Num. inference steps & 2 \\
& Num. train steps & 100 \\
\bottomrule
\end{tabular}
\end{table}

\subsection{Details of Real-World Experiment Settings}
The real-world experimental setup is shown in Fig.~\ref{fig:task}, and the task workflow is illustrated in Fig.~\ref{fig:vis}. We design three tasks: Place Object, Adjust Bottle, and Stack Blocks Two. Specifically, the three tasks proceed as follows:
\begin{itemize}
    \item \texttt{Place Object}: The robot picks up the target object and places it onto the designated platform.
    \item \texttt{Adjust Bottle}: The robot grasps the fallen cup and then sets it upright.
    \item \texttt{Stack Blocks Two}: The robot grasps one block and stacks it on top of another block.
\end{itemize}

\begin{figure}[htbp]
    \centering
    \includegraphics[width=0.8\linewidth]{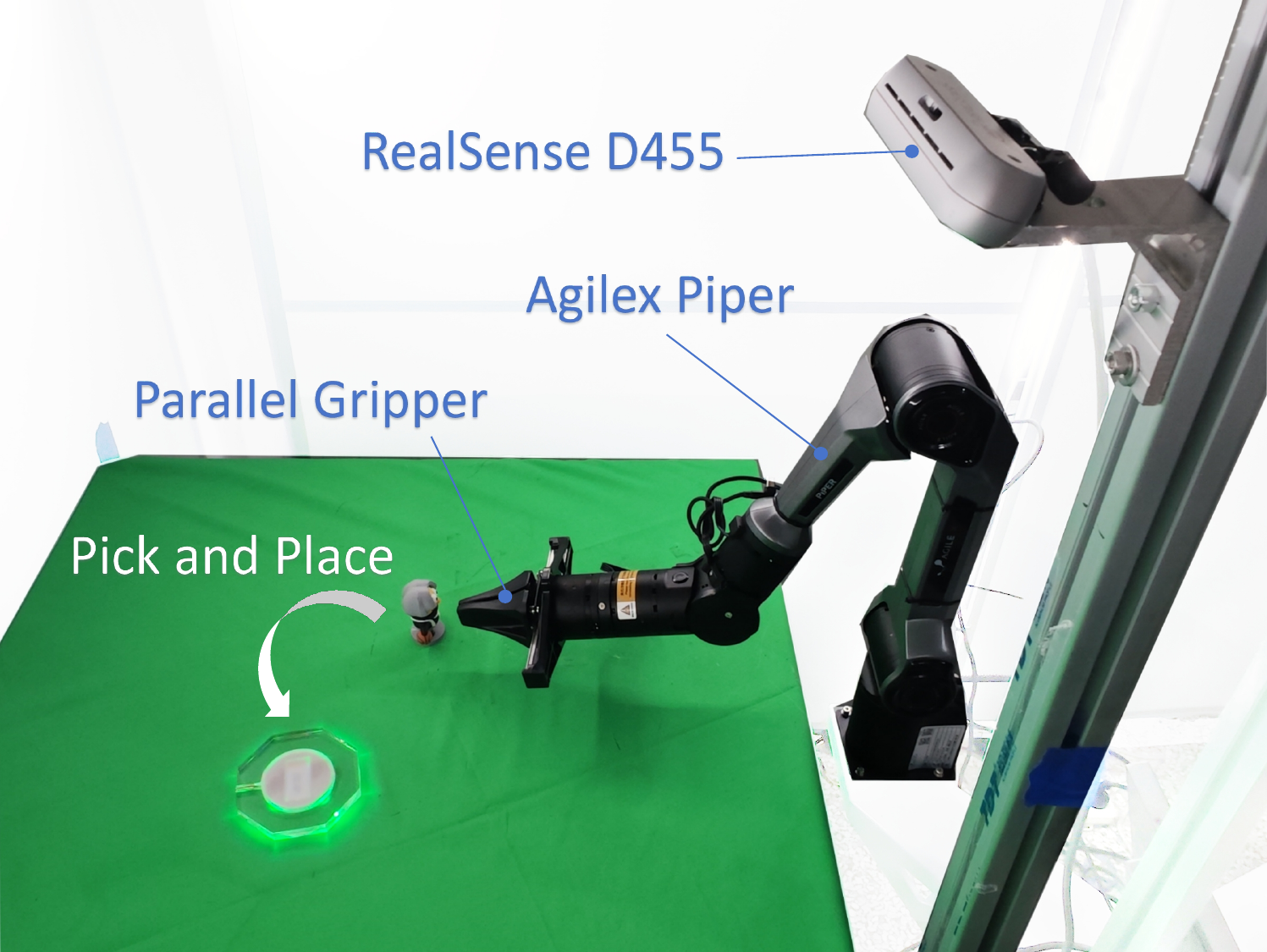}
    \caption{\textbf{Real-World Experiment Setup.} }
    \label{fig:task}
    \vspace{-3.0mm}
\end{figure}

\begin{figure}[ht]
  \centering
  \includegraphics[width=\linewidth]{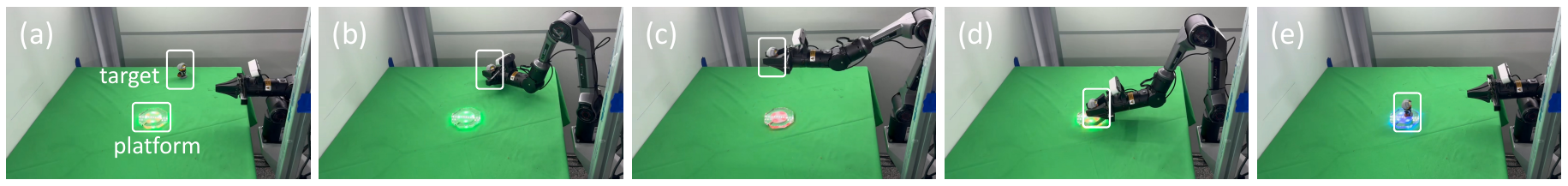} \\
  \vspace{1pt} 
  
  \includegraphics[width=\linewidth]{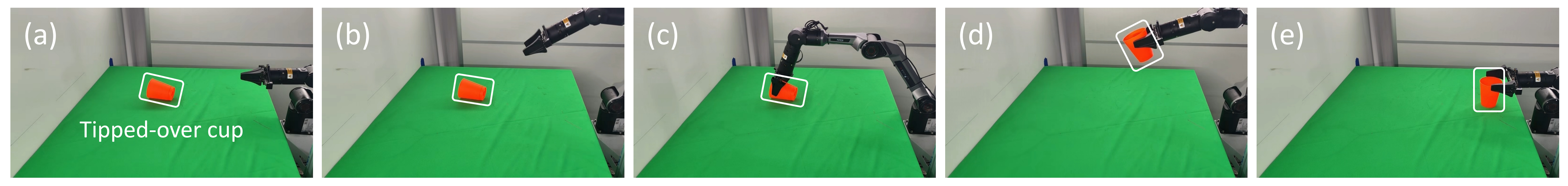} \\
  \vspace{1pt}
  
  \includegraphics[width=\linewidth]{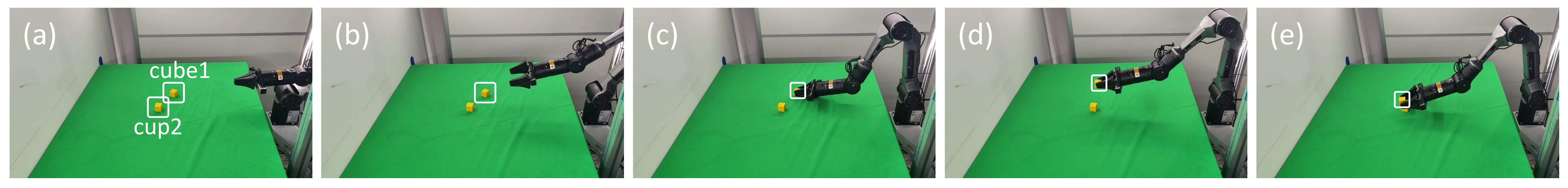}

  \caption{\textbf{Real-world Experiments.} The image sequence (top to bottom) illustrates the robot successfully performing three tasks: placing an object, uprighting a fallen cup, and stacking two blocks.}
  \label{fig:vis}
  \vspace{-2.0mm}
\end{figure}

\section{Frequency-domain Decomposition Details}\label{app:freq}
\begin{figure}[!t]
    \centering
    \includegraphics[width=0.8\linewidth]{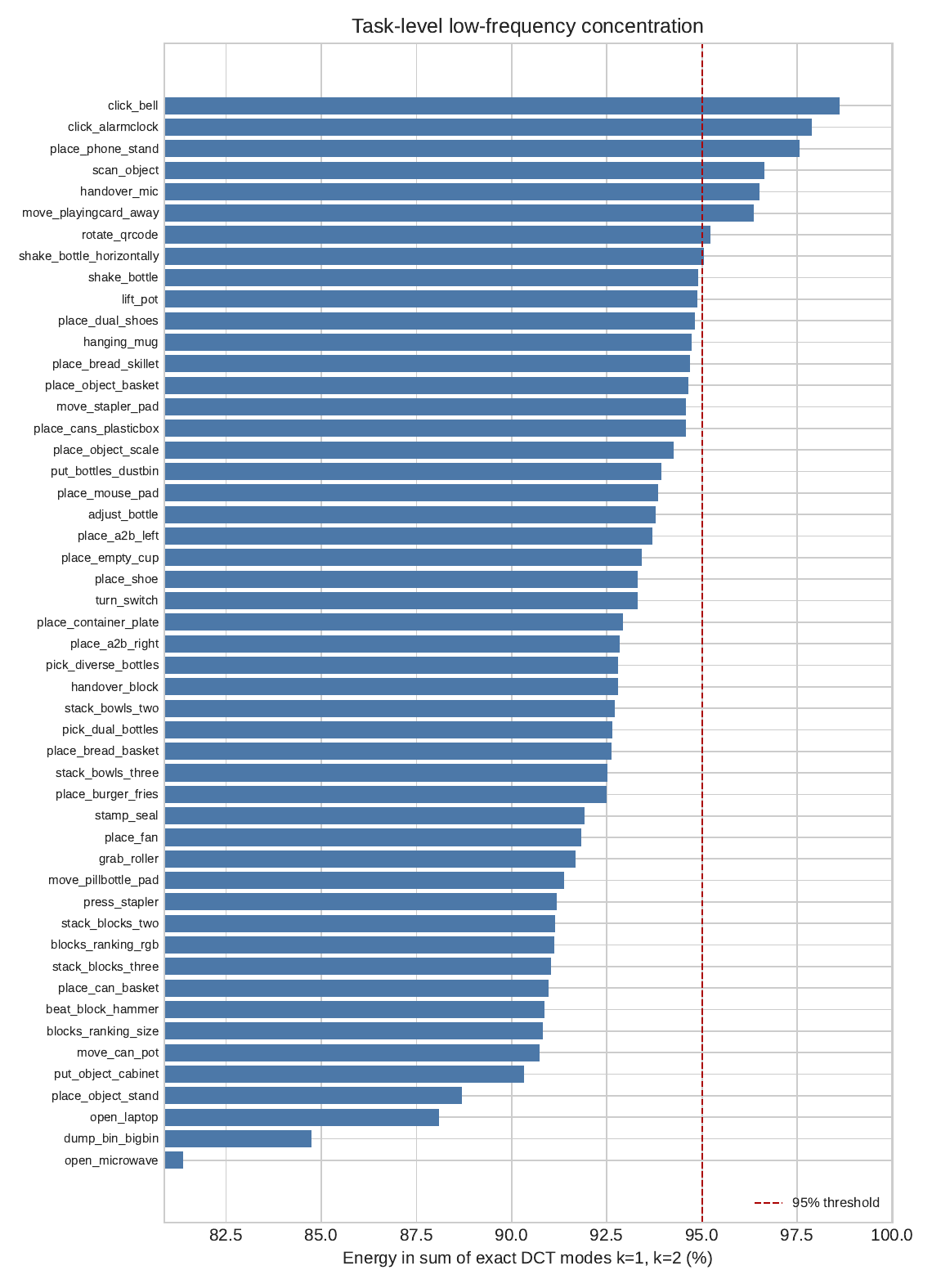}
    \caption{Fraction of Energy Contained in the First 5\% of DCT Modes for Each Task}
    \label{fig:dct_et}
    \vspace{-3.0mm}
\end{figure}

For each trajectory segment, we analyze the $14$-dimensional action sequence $X\in\mathbb{R}^{T\times 14}$, where $X_{t,d}$ denotes the $d$-th action component at time step $t$. Each episode is partitioned into non-overlapping windows of length $T$ (e.g., $T=32$ or $T=8$), and any remainder shorter than $T$ is discarded. For each segment, we first remove the per-dimension temporal mean,
\[
\tilde{X}_{t,d}=X_{t,d}-\frac{1}{T}\sum_{\tau=0}^{T-1}X_{\tau,d},
\]
We then apply a one-dimensional DCT-II independently along the time axis of each action dimension, using orthonormal normalization:
\[
C_{k,d}=\alpha_k\sum_{t=0}^{T-1}\tilde{X}_{t,d}\cos\!\left[\frac{\pi}{T}\left(t+\frac{1}{2}\right)k\right],\qquad k=0,\dots,T-1,
\]
with $\alpha_0=T^{-1/2}$ and $\alpha_k=(2/T)^{1/2}$ for $k\geq 1$. Because this transform is orthonormal, Parseval's identity holds, and the squared coefficients can be interpreted directly as energy. We therefore define the energy of mode $k$ for a segment as
\[
E_k=\sum_{d=1}^{14} C_{k,d}^2,
\]
and its normalized energy share as
\[
p_k=\frac{E_k}{\sum_{j=0}^{T-1}E_j}.
\]
Dataset-level spectra are obtained by summing $E_k$ over all segments and then normalizing by the total summed energy. We report both exact mode-wise energy shares $\{p_k\}_{k=0}^{T-1}$ in Sec.~\ref{sec:motivation} and fraction of energy contained in the first 5\%($k=1,2$) of DCT modes for each task in Fig.~\ref{fig:dct_et}. As can be seen, the first two DCT modes account for more than 80\% of the total energy for every task. Among the 50 tasks, only 4 have a cumulative energy ratio below 90\% for these two modes. On average, the (k=1) mode alone accounts for 93.2\% of the total energy, indicating that robot trajectories are highly concentrated in low-frequency components.

\section{Detailed Experimental Results on RoboTwin2.0}
As shown in Table~\ref{tab:robotwin2_pocketdp3_large_full}, we report the comparison of model success rates across 50 tasks on RoboTwin2.0.
\begin{table*}[h]
\small
\setlength{\tabcolsep}{3.6pt}
\renewcommand{\arraystretch}{1.12}

\resizebox{\textwidth}{!}{%
\begin{tabular}{@{}l*{10}{c}@{}}
\toprule
~ & Adjust Bottle & Beat Block Hammer & Blocks Ranking RGB & Blocks Ranking Size & Click Alarmclock & Click Bell & Dump Bin Bigbin & Grab Roller & Handover Block & Handover Mic \\
\midrule
DP & 97.0 & 42.0 & 0.0 & 1.0 & 61.0 & 54.0 & 49.0 & \underline{98.0} & 10.0 & 53.0 \\
DP3 & \underline{99.0} & 72.0 & \textbf{3.0} & 2.0 & 77.0 & 90.0 & \textbf{85.0} & \underline{98.0} & \underline{70.0} & \textbf{100.0} \\
Flow Policy & \textbf{100.0} & \underline{75.0} & 0.0 & 0.0 & \textbf{96.0} & \textbf{100.0} & 80.0 & \underline{98.0} & 16.0 & 67.0 \\
MP1 & 95.0 & 63.0 & \underline{2.0} & \underline{3.0} & \underline{91.0} & \underline{98.0} & 83.0 & \textbf{100.0} & \textbf{77.0} & \underline{93.0} \\
\rowcolor{oursblue}
HDP3 \textbf{(Ours)} & 92.0 & \textbf{88.0} & \textbf{3.0} & \textbf{8.0} & \textbf{96.0} & \textbf{100.0} & \underline{84.0} & \textbf{100.0} & \underline{70.0} & \textbf{100.0} \\
\bottomrule
\end{tabular}%
}
\vspace{3pt}

\resizebox{\textwidth}{!}{%
\begin{tabular}{@{}l*{10}{c}@{}}
\toprule
~ & Hanging Mug & Lift Pot & Move Can Pot & Move Pillbottle Pad & Move Playingcard Away & Move Stapler Pad & Open Laptop & Open Microwave & Pick Diverse Bottles & Pick Dual Bottles \\
\midrule
DP & 8.0 & 39.0 & 39.0 & 1.0 & 47.0 & 1.0 & 49.0 & 5.0 & 6.0 & 24.0 \\
DP3 & \underline{17.0} & \textbf{97.0} & 70.0 & \underline{41.0} & 68.0 & \textbf{12.0} & \underline{82.0} & \underline{61.0} & 52.0 & 60.0 \\
Flow Policy & 6.0 & 48.0 & 40.0 & 14.0 & 64.0 & 1.0 & 76.0 & 7.0 & 64.0 & 83.0 \\
MP1 & 11.0 & 81.0 & \underline{87.0} & 35.0 & \textbf{76.0} & 5.0 & 70.0 & 24.0 & \textbf{73.0} & \textbf{93.0} \\
\rowcolor{oursblue}
HDP3 \textbf{(Ours)} & \textbf{35.0} & \underline{95.0} & \textbf{95.0} & \textbf{56.0} & \underline{73.0} & \underline{10.0} & \textbf{85.0} & \textbf{92.0} & \underline{71.0} & \underline{90.0} \\
\bottomrule
\end{tabular}%
}
\vspace{3pt}

\resizebox{\textwidth}{!}{%
\begin{tabular}{@{}l*{10}{c}@{}}
\toprule
~ & Place A2B Left & Place A2B Right & Place Bread Basket & Place Bread Skillet & Place Burger Fries & Place Can Basket & Place Cans Plasticbox & Place Container Plate & Place Dual Shoes & Place Empty Cup \\
\midrule
DP & 2.0 & 13.0 & 14.0 & 11.0 & 72.0 & 18.0 & 40.0 & 41.0 & 8.0 & 37.0 \\
DP3 & 46.0 & \textbf{49.0} & 26.0 & 19.0 & 72.0 & \underline{67.0} & 48.0 & \underline{86.0} & \underline{13.0} & 65.0 \\
Flow Policy & 37.0 & 22.0 & 18.0 & 30.0 & 55.0 & 17.0 & 18.0 & 80.0 & 7.0 & 58.0 \\
MP1 & \underline{48.0} & \underline{29.0} & \textbf{40.0} & \underline{39.0} & \underline{84.0} & 58.0 & \underline{86.0} & \textbf{87.0} & \textbf{14.0} & \underline{83.0} \\
\rowcolor{oursblue}
HDP3 \textbf{(Ours)} & \textbf{50.0} & 27.0 & \underline{39.0} & \textbf{40.0} & \textbf{90.0} & \textbf{78.0} & \textbf{96.0} & \textbf{87.0} & 8.0 & \textbf{89.0} \\
\bottomrule
\end{tabular}%
}
\vspace{3pt}

\resizebox{\textwidth}{!}{%
\begin{tabular}{@{}l*{10}{c}@{}}
\toprule
~ & Place Fan & Place Mouse Pad & Place Object Basket & Place Object Scale & Place Object Stand & Place Phone Stand & Place Shoe & Press Stapler & Put Bottles Dustbin & Put Object Cabinet \\
\midrule
DP & 3.0 & 0.0 & 15.0 & 1.0 & 22.0 & 13.0 & 23.0 & 6.0 & 22.0 & 42.0 \\
DP3 & \underline{36.0} & \underline{4.0} & \textbf{65.0} & \textbf{15.0} & 60.0 & 44.0 & \textbf{58.0} & 69.0 & \underline{60.0} & \textbf{72.0} \\
Flow Policy & 9.0 & 2.0 & 21.0 & 6.0 & 51.0 & 44.0 & 39.0 & \textbf{92.0} & 0.0 & 7.0 \\
MP1 & 27.0 & 2.0 & \underline{53.0} & \underline{14.0} & \underline{61.0} & \underline{52.0} & 40.0 & 64.0 & 10.0 & 48.0 \\
\rowcolor{oursblue}
HDP3 \textbf{(Ours)} & \textbf{39.0} & \textbf{7.0} & \textbf{65.0} & 9.0 & \textbf{65.0} & \textbf{65.0} & \underline{57.0} & \underline{73.0} & \textbf{64.0} & \underline{63.0} \\
\bottomrule
\end{tabular}%
}
\vspace{3pt}

\resizebox{\textwidth}{!}{%
\begin{tabular}{@{}l*{10}{c}@{}}
\toprule
~ & Rotate QRcode & Scan Object & Shake Bottle & Shake Bottle Horizontally & Stack Blocks Three & Stack Blocks Two & Stack Bowls Three & Stack Bowls Two & Stamp Seal & Turn Switch \\
\midrule
DP & 13.0 & 9.0 & 65.0 & 59.0 & 0.0 & 7.0 & \underline{63.0} & 61.0 & 2.0 & 36.0 \\
DP3 & \textbf{74.0} & 31.0 & \underline{98.0} & \textbf{100.0} & \underline{1.0} & 24.0 & 57.0 & \underline{83.0} & 18.0 & 46.0 \\
Flow Policy & 14.0 & 35.0 & \underline{98.0} & 93.0 & 0.0 & 13.0 & 0.0 & 73.0 & 20.0 & 50.0 \\
MP1 & 44.0 & \underline{41.0} & 97.0 & \underline{98.0} & \underline{1.0} & \underline{34.0} & 50.0 & 77.0 & \underline{28.0} & \textbf{62.0} \\
\rowcolor{oursblue}
HDP3 \textbf{(Ours)} & \underline{55.0} & \textbf{45.0} & \textbf{100.0} & \textbf{100.0} & \textbf{2.0} & \textbf{41.0} & \textbf{74.0} & \textbf{89.0} & \textbf{43.0} & \underline{55.0} \\
\bottomrule
\end{tabular}%
}
\centering
\caption{Evaluation on the Robotwin2.0 benchmark. Each task is tested across 100 randomly generated scenes.}
\label{tab:robotwin2_pocketdp3_large_full}
\end{table*}

\section{Limitations} \label{app:lim}
Although we validate the efficiency of our architecture on 3D-input visuomotor policies, its effectiveness on other input modalities remains to be investigated. In addition, the precision of our theoretical analysis still leaves room for further improvement.

\section{Broader Impacts}
\label{app:broader}

This work studies how to make diffusion-based visuomotor policies smaller and faster while maintaining strong manipulation performance. A potential positive impact is that more efficient policies can lower the hardware and energy cost of robotics research and deployment, making strong visuomotor control methods more accessible to smaller labs and resource-constrained platforms. In addition, faster inference can improve the practicality of closed-loop robot control in real-world manipulation settings where latency matters.

At the same time, more efficient manipulation policies can also introduce risks if they are deployed without adequate safeguards. In particular, visuomotor policies may fail under distribution shift, unexpected contacts, sensor noise, or task conditions that differ from those seen in training. Such failures can lead to unsafe motions, unintended collisions, or damage to surrounding objects. More broadly, improved manipulation capability may accelerate automation in labor-sensitive settings or be adapted to applications that require stronger safety, reliability, or accountability guarantees than those studied here.

We therefore emphasize that our experiments are limited to benchmark environments and controlled laboratory settings, and we do not claim that the method is ready for safety-critical deployment. We encourage future users to combine policies of this kind with standard robotic safety measures, including simulation-based validation before deployment, conservative workspace and collision constraints, human oversight during testing, and hardware-level protections such as emergency stops and motion limits.

\section{External Assets, Licenses, and Planned Release}
\label{app:assets}

Our experiments rely on several existing benchmarks and open-source codebases. For RoboTwin2.0, we use the official RoboTwin repository and benchmark assets; the official repository is released under the MIT license. For MetaWorld, we use the official MetaWorld benchmark implementation, which is also released under the MIT license. For Adroit, we follow the standard D4RL/Adroit benchmark assets; the D4RL codebase is released under the Apache-2.0 license, and the benchmark maintainers state that, unless otherwise noted, the datasets are released under CC BY 4.0.

For baseline implementations, we use official open-source code when available and cite the corresponding original papers in the bibliography. In particular, the official public repositories for Diffusion Policy, 3D Diffusion Policy (DP3), and FlowPolicy are released under the MIT license. We follow the documented licenses and terms of use for these assets and will include the exact repository versions, package versions, or commit identifiers used in our experiments in the supplemental material.

To support reproducibility, we will provide anonymized code and documentation alongside the submission's supplemental material. The release package will include environment setup instructions, exact commands for training and evaluation, and a license file for the released code.



\end{document}